%% file: main.tex
\def\year{2021}\relax
\documentclass[letterpaper]{article} 
\usepackage{aaai21}  
\usepackage{times}  
\usepackage{helvet} 
\usepackage{courier}  
\usepackage[hyphens]{url}  
\usepackage{graphicx} 
\usepackage{natbib}  
\usepackage{caption} 
\usepackage{geometry}
\usepackage[utf8]{inputenc} 
\usepackage{booktabs}       
\usepackage{amsfonts}       
\usepackage{nicefrac}       
\usepackage{microtype}      
\usepackage{amssymb, amsthm, mathtools}
\usepackage{amsmath}
\usepackage{subfigure}
\usepackage{physics}
\usepackage{array, multirow}
\usepackage{calc}
\usepackage{tabularx}
\usepackage{xspace}
\usepackage{enumitem}
\usepackage{mathdots}
\usepackage{color, colortbl}
\usepackage{textcomp}
\usepackage{changepage}
\usepackage{placeins}
\usepackage{tikz-cd}
\usepackage{bbm}
\usepackage{algorithm}
\usepackage{stmaryrd}
\usepackage{algpseudocode}
\usepackage{matsup}

\newtheorem{innercustomgeneric}{\customgenericname}
\providecommand{\customgenericname}{}
\newcommand{\newcustomtheorem}[2]{%
  \newenvironment{#1}[1]
  {%
   \renewcommand\customgenericname{#2}%
   \renewcommand\theinnercustomgeneric{##1}%
   \innercustomgeneric
  }
  {\endinnercustomgeneric}
}

\newtheorem{theorem}{Theorem}

\newtheorem{lemma}{Lemma}

\newtheorem*{definition*}{Definition}
\newtheorem*{theorem*}{Theorem}
\newtheorem*{lemma*}{Lemma}
\newcustomtheorem{customtheorem}{Theorem}
\newcustomtheorem{customlemma}{Lemma}

\DeclareMathOperator*{\reshape}{vec}
\DeclareMathOperator*{\pad}{pad}

\DeclareMathOperator*{\LipBound}{LipBound}

\newcommand{\ci}{\mathbf{i}}
\newcommand{\leftmat}{\left(} 
\newcommand{\rightmat}{\right)}
\newcommand{\seqidx}{h}
\newcommand{\seqsetN}{N}
\newcommand{\seqsetM}{M}
\newcommand{\cin}{cin}
\newcommand{\cout}{cout}
\newcommand{\eg}{\textit{e.g}\xspace}
\newcommand{\ie}{\textit{i.e.}\xspace}

\title{On Lipschitz Regularization of Convolutional Layers \\ using Toeplitz Matrix Theory}

\author{
  Alexandre Araujo,
  Benjamin Negrevergne,
  Yann Chevaleyre,
  Jamal Atif \\
}

\affiliations{
  PSL, Université Paris-Dauphine, \\ CNRS,\ LAMSADE,\ MILES Team,\ Paris,\ France \\
}

\newenvironment{testenv}{
  \newpage
  \setlength{\parindent}{0pt}
  \newgeometry{
    textheight=9in,
    textwidth=7in,
    top=1in,
    headheight=12pt,
    headsep=25pt,
    footskip=30pt
  }
  \onecolumn
}{\restoregeometry}

\begin{document}

  \maketitle

\input{paper}

  \bibliography{bibliography}

  \begin{testenv}
    \input{matsup}

  \end{testenv}

\end{document}

%% file: paper.tex
\begin{abstract}
This paper tackles the problem of Lipschitz regularization of Convolutional Neural Networks.
Lipschitz regularity is now established as a key property of modern deep learning with implications in training stability, generalization, robustness against adversarial examples, etc.
However, computing the exact value of the Lipschitz constant of a neural network is known to be NP-hard.
Recent attempts from the literature introduce upper bounds to approximate this constant that are either efficient but loose or accurate but computationally expensive.
In this work, by leveraging the theory of Toeplitz matrices, we introduce a new upper bound for convolutional layers that is both tight and easy to compute.
Based on this result we devise an algorithm to train Lipschitz regularized Convolutional Neural Networks.
\end{abstract}

\section{Introduction}
\label{section:introduction}

The last few years have witnessed a growing interest in Lipschitz regularization of neural networks, with the aim of improving their generalization~\cite{bartlett2017spectrally}, their robustness to adversarial attacks~\cite{tsuzuku2018lipschitz, farnia2018generalizable}, or their generation abilities (\eg for GANs: \citealt{miyato2018spectral,arjovsky2017wasserstein}).
Unfortunately computing  the exact Lipschitz constant of a neural network is NP-hard~\cite{scaman2018lipschitz} and in practice, existing techniques such as~\citet{scaman2018lipschitz, NIPS2019_9319} or~\citet{latorre2020lipschitz} are difficult to implement for neural networks with more than one or two layers, which hinders their use in deep learning applications.

To overcome this difficulty, most of the work has focused on computing the Lipschitz constant of \emph{individual layers} instead.
The product of the Lipschitz constant of each layer is an upper-bound for the Lipschitz constant of the entire network, and it can be used as a surrogate to perform Lipschitz regularization.
Since most common activation functions (such as ReLU) have a Lipschitz constant equal to one, the main bottleneck is to compute the Lipschitz constant of the underlying linear application which is equal to its maximal singular value.
The work in this line of research mainly relies on the celebrated iterative algorithm by~\citet{golub2000eigenvalue} used to approximate the maximum singular value of a linear function.
Although generic and accurate, this technique is also computationally expensive, which impedes its usage in large training settings. 

In this paper we introduce a new upper bound on the largest singular value of convolution layers that is both tight and easy to compute.
Instead of using the power method to iteratively approximate this value, we rely on Toeplitz matrix theory and its links with Fourier analysis.
Our work is based on the result \cite{gray2006toeplitz} that an upper bounded on the singular value of  Toeplitz matrices can be computed from 
the inverse Fourier transform of the characteristic sequence of these matrices.
We first extend this result to doubly-block Toeplitz matrices (\ie, block Toeplitz matrices where each block is Toeplitz) and then to convolutional operators, which can be represented as stacked sequences of doubly-block Toeplitz matrices. 
From our analysis immediately follows an algorithm for bounding the Lipschitz constant of a convolutional layer, and by extension the Lipschitz constant of the whole network. We theoretically study the approximation of this algorithm and show experimentally that it is more efficient and accurate than competing approaches.

Finally, we illustrate our approach on adversarial robustness. 
Recent work has shown that empirical methods such as adversarial training (AT) offer poor generalization~\cite{schmidt2018adversarially}, and can be improved by applying Lipschitz regularization~\cite{farnia2018generalizable}.
To illustrate the benefit of our new method, we train a large, state-of-the-art Wide ResNet architecture with Lipschitz regularization and show that it offers a significant improvement over adversarial training alone, and over other methods for Lipschitz regularization. 
In summary, we make the three following contributions:

\begin{enumerate}
  \item We devise an upper bound on the singular values of the operator matrix of convolutional layers by leveraging Toeplitz matrix theory and its links with Fourier analysis.
  \item We propose an efficient algorithm to compute this upper bound which enables its use in the context of Convolutional Neural Networks.
  \item We use our method to regularize the Lipschitz constant of neural networks for adversarial robustness and show that it offers a significant improvement over AT alone.
\end{enumerate}

\section{Related Work}
\label{section:related_work}

A popular technique for approximating the maximal singular value of a matrix is the power method~\cite{golub2000eigenvalue}, an iterative algorithm which yields a good approximation of the maximum singular value when the algorithm is able to run for a sufficient number of iterations.

\citet{yoshida2017spectral, miyato2018spectral} have used the power method to normalize the spectral norm of each layer of a neural network, and showed that the resulting models offered improved generalization performance and generated better examples when they were used in the context of GANs. 
\citealt{farnia2018generalizable} built upon the work of ~\citet{miyato2018spectral} and proposed a power method specific for convolutional layers that leverages the deconvolution operation and avoid the computation of the gradient.
They used it in combination with adversarial training. 
In the same vein, \citet{gouk2018regularisation} demonstrated that regularized neural networks using the power method also offered improvements over their non-regularized counterparts. 
Furthermore, \citet{tsuzuku2018lipschitz} have shown that a neural network can be more robust to some adversarial attacks,  if the prediction margin of the network (\ie, the difference between the first and the second maximum logit) is higher than a minimum threshold that depends on the global Lipschitz constant of the network.
Building on this observation, they use the power method to compute an upper bound on the global Lipschitz constant, and maximize the prediction margin during training.
Finally, \citet{scaman2018lipschitz} have used automatic differentiation combined with the power method to compute a tighter bound on the global Lipschitz constant of neural networks.
Despite a number of interesting results, using the power method is expensive and results in prohibitive training times. 

Other approaches to regularize the Lipschitz constant of neural networks have been proposed by~\citet{sedghi2018the} and ~\citet{singla2019bounding}.
The method of~\citet{sedghi2018the} exploits the properties of circulant matrices to approximate the maximal singular value of a convolutional layer.
Although interesting, this method results in a loose approximation of the maximal singular value of a convolutional layer.
Furthermore, the complexity of their algorithm is dependent on the convolution input which can be high for large datasets such as ImageNet.
More recently, \citet{singla2019bounding} have successfully bounded the operator norm of the Jacobian matrix of a convolution layer by the Frobenius norm of the reshaped kernel.
This technique has the advantage to be very fast to compute and to be independent of the input size but it also results in a loose approximation. 

To build robust neural networks, \citet{cisse2017parseval} and ~\citet{NIPS2019_9673} have proposed to constrain the Lipschitz constant of neural networks by using orthogonal convolutions.
\citet{cisse2017parseval} use the concept of \emph{parseval tight frames}, to constrain their networks.
\citet{NIPS2019_9673} built upon the work of~\citet{cisse2017parseval} to propose an efficient construction method of orthogonal convolutions.  
Also, recent work~\cite{NIPS2019_9319,latorre2020lipschitz} has proposed a tight bound on the Lipschitz constant of the full network with the use of semi-definite programming.
These works are theoretically interesting but lack scalability (\ie, the bound can only be computed on small networks).

Finally, in parallel to the development of the results in this paper, we discovered that \citet{yi2020asymptotic} have studied the asymptotic distribution of the singular values of convolutional layers by using a related approach. However, this author does not investigate the robustness applications of Lipschitz regularization.

\section{A Primer on Toeplitz and block Toeplitz matrices}
\label{section:primer_toeplitz_matrix}

In order to devise a bound on the Lipschitz constant of a convolution layer as used by the Deep Learning community, we study the properties of doubly-block Toeplitz matrices.
In this section, we first introduce the necessary background on Toeplitz and block Toeplitz matrices, and introduce a new result on doubly-block Toeplitz matrices.

Toeplitz matrices and block Toeplitz matrices are well-known types of structured matrices.
A Toeplitz matrix  (respectively a block Toeplitz matrix) is a matrix in which each scalar (respectively block) is repeated identically along diagonals.

An $n\times n$ Toeplitz matrix $\mathbf A$ is fully determined by a two-sided sequence of scalars: $\{a_\seqidx\}_{\seqidx \in \seqsetN}$, whereas an $nm\times nm$ block Toeplitz matrix $\mathbf B$ is fully determined by a two-sided sequence of blocks $\{\mathbf B_\seqidx\}_{\seqidx \in \seqsetN}$, where $\seqsetN = \{-n+1,\dots, n-1 \}$ and where each block $\mathbf B_\seqidx$ is an $m \times m$ matrix.

{\small
\begin{equation*}
    \mathbf{A} = \begin{psmallmatrix}
      a_0 & a_{1} & \cdots & a_{n-1} \\ \vspace{0.1cm}
      a_{-1} & a_0 & \ddots & \vdots \\ \vspace{0.3cm}
     \vdots & \ddots & a_{0} & a_{1} \\ 
    a_{-n+1} & \cdots  & a_{-1}    & a_0 
    \end{psmallmatrix} \quad
    \mathbf{B} = \begin{psmallmatrix}
      \mathbf{B}_0 & \mathbf{B}_{1} & \cdots & \mathbf{B}_{n-1} \\ \vspace{0.1cm}
      \mathbf{B}_{-1} & \mathbf{B}_0 & \ddots & \vdots \\ \vspace{0.3cm}
     \vdots & \ddots & \mathbf{B}_0 & \mathbf{B}_{1} \\ 
    \mathbf{B}_{-n+1} & \cdots  & \mathbf{B}_{-1}    & \mathbf{B}_0 
    \end{psmallmatrix}.
\end{equation*}
}

Finally, a doubly-block Toeplitz matrix is a block Toeplitz matrix in which each block is itself a Toeplitz matrix.
In the remainder, we will use the standard notation $\leftmat\ \cdot\ \rightmat_{i,j\in\{0,\ldots,n-1\}}$ to construct (block) matrices.
For example, $\mathbf{A} = \leftmat a_{j-i} \rightmat_{i,j \in \{0, \ldots, n-1 \}}$ and $\mathbf{B} = \leftmat \mathbf{B}_{j-i} \rightmat_{i,j \in \{0, \ldots, n-1\}}$.

\subsection{Bound on the singular value of Toeplitz and block Toeplitz matrices}
\label{subsection:generating_function}

A standard tool for manipulating (block) Toeplitz matrices is the use of Fourier analysis.
Let $\{a_\seqidx\}_{\seqidx \in \seqsetN}$ be the sequence of coefficients of the Toeplitz matrix $\mathbf{A} \in \mathbb{R}^{n\times n}$ and let $\{\mathbf B_\seqidx\}_{\seqidx \in \seqsetN}$ be the sequence of $m\times m$ blocks of the block Toeplitz matrix $\mathbf{B}$.
The complex-valued function $f(\omega) = \sum_{\seqidx \in \seqsetN} a_\seqidx e^{\ci \seqidx \omega}$ and the matrix-valued function $F(\omega) = \sum_{\seqidx \in \seqsetN} \mathbf{B}_\seqidx e^{\ci \seqidx \omega}$ are the \emph{inverse Fourier transforms} of the sequences $\{a_\seqidx\}_{\seqidx \in \seqsetN}$ and $\{\mathbf B_\seqidx\}_{\seqidx \in \seqsetN}$, with $\omega \in \mathbb{R}$.
From these two functions, one can recover these two sequences using the standard Fourier transform:
{\small
\begin{equation}
  a_\seqidx = \frac{1}{2\pi} \int_0^{2\pi} e^{-\ci \seqidx \omega} f(\omega) d\omega \quad \mathbf{B}_\seqidx = \frac{1}{2\pi} \int_0^{2\pi} e^{-\ci \seqidx \omega} F(\omega) d\omega .
\end{equation}}
From there, similarly to the work done by~\citet{gray2006toeplitz} and~\citet{gutierrez2012block}, we can define an operator $\mathbf{T}$ mapping integrable functions to matrices:
\begin{equation}
  \mathbf{T}(g)  \triangleq\leftmat\frac{1}{2\pi}\int_{0}^{2\pi}e^{-\ci(i-j)\omega}g(\omega)\,d\omega\rightmat_{i,j\in\{0,\ldots,n-1\}} .
  \label{equation:expression_toeplitz_matrix}
\end{equation}

Note that if $f$ is the inverse Fourier transform of $\{a_\seqidx\}_{\seqidx \in \seqsetN}$, then $\mathbf{T}(f)$ is equal to $\mathbf{A}$.
Also, if $F$ is the inverse Fourier transform of $\{\mathbf B_\seqidx\}_{\seqidx \in \seqsetN}$ as defined above, then the integral in Equation~\ref{equation:expression_toeplitz_matrix} is matrix-valued, and thus $\mathbf{T}(F) \in \mathbb{R}^{mn \times mn}$ is the block matrix $\mathbf{B}$.
Now, we can state two known theorems which upper bound the maximal singular value of Toeplitz and block Toeplitz matrices with respect to their generating functions.
In the rest of the paper, we refer to $\sigma_1(\ \cdot\ )$ as the maximal singular value. 
\begin{theorem}[Bound on the singular values of Toeplitz matrices] \label{theorem:teoplitz_sup_singular}
  Let $f: \mathbb{R} \rightarrow \mathbb{C}$, be continuous and $2\pi$-periodic. Let $\mathbf{T}(f) \in \mathbb{R}^{n \times n}$ be a Toeplitz matrix generated by the function $f$, then:
  \begin{equation}
    \sigma_1\left(\mathbf{T}(f)\right) \leq \sup_{\omega \in [0, 2\pi]} |f(\omega)|.
  \end{equation}
\end{theorem}
\noindent
Theorem~\ref{theorem:teoplitz_sup_singular} is a direct application of Lemma 4.1 in~\citet{gray2006toeplitz} for real Toeplitz matrices. 
\begin{theorem}[Bound on the singular values of Block Toeplitz matrices ~\cite{gutierrez2012block}] \label{theorem:block_teoplitz_sup_singular}
    Let $F: \mathbb{R} \rightarrow \mathbb{C}^{m \times m}$ be a matrix-valued function which is continuous and $2 \pi$-periodic. Let $\mathbf{T}(F) \in \mathbb{R}^{mn \times mn}$ be a block Toeplitz matrix generated by the function $F$, then:
  \begin{equation}
    \sigma_1\left(\mathbf{T}(F)\right) \leq \sup_{\omega \in [0, 2\pi]} \sigma_1(F\left(\omega)\right) .
  \end{equation}
\end{theorem}

\subsection{Bound on the singular value of Doubly-Block Toeplitz matrices}
\label{section:bound_singular_value_doubly_block_toeplitz}

We extend the reasoning from Toeplitz and block Toeplitz matrices to doubly-block Toeplitz matrices (\ie, block Toeplitz matrices where each block is also a Toeplitz matrix).
A doubly-block Toeplitz matrix can be generated by a function $f: \mathbb{R}^2 \rightarrow \mathbb{C}$ using the 2-dimensional inverse Fourier transform.
For this purpose, we define an operator $\mathbf{D}$ which maps a function $f: \mathbb{R}^2 \rightarrow \mathbb{C}$ to a doubly-block Toeplitz matrix of size $nm \times nm$.
For the sake of clarity, the dependence of $\mathbf{D}(f)$  on $m$ and $n$ is omitted.
Let $\mathbf{D}(f) =\leftmat\mathbf{D}_{i,j}(f)\rightmat_{i,j\in\{0,\ldots ,n-1\}}$ where $\mathbf{D}_{i,j}$ is defined as:
{\small
\begin{equation}
  \mathbf{D}_{i,j}(f) = \leftmat\frac{1}{4\pi^{2}}\int_{\Omega^{2}}e^{-\mathbf{i} \psi}f(\omega_{1},\omega_{2})\,d(\omega_{1},\omega_{2})\rightmat_{k,l\in\{0,\ldots,m-1\}}
  \label{equation:doubly_block_toeplitz_operator}
\end{equation}}
where $\Omega = [0, 2\pi]$ and $\psi = (i - j) \omega_{1} + (k - l) \omega_{2}$.
	
We are now able to combine Theorem~\ref{theorem:teoplitz_sup_singular} and Theorem~\ref{theorem:block_teoplitz_sup_singular} to bound the maximal singular value of doubly-block Toeplitz matrices with respect to their generating functions. 

\begin{theorem}[Bound on the Maximal Singular Value of a Doubly-Block Toeplitz Matrix] \label{theorem:doubly_block_teoplitz_sup_singular}
  Let $\mathbf{D}(f) \in \mathbb{R}^{nm \times nm}$ be a doubly-block Toeplitz matrix generated by the function $f$, then:
  \begin{equation}
    \sigma_{1} \left( \mathbf{D}(f) \right) \leq \sup_{\omega_1, \omega_2 \in [0, 2\pi]^2}|f(\omega_1,\omega_2)|
  \end{equation}
  where the function $f: \mathbb{R}^2 \rightarrow \mathbb{C}$, is a multivariate trigonometric polynomial of the form:
  \begin{equation} \label{equation:muli_variate_poly_on_M}
    f(\omega_1, \omega_2) \triangleq \sum_{h_1 \in \seqsetN} \sum_{h_2 \in \seqsetM} d_{h_1, h_2} e^{\ci (h_1 \omega_1 + h_2 \omega_2)},
  \end{equation}
  where $d_{h_{1},h_{2}}$ is the ${h_2}^\textrm{th}$ scalar of the ${h_1}^\textrm{th}$ block of the doubly-Toeplitz matrix $\mathbf{D}(f)$, and where $\seqsetM = \{-m+1,\dots, m-1 \}$.
\end{theorem}

\section{Bound on the Singular Values of Convolutional Layers}
\label{section:bound_lipschitz_cst_conv}

From now on, without loss of generality, we will assume that $n=m$ to simplify notations.
It is well known that a discrete convolution operation with a 2d kernel applied on a 2d signal is equivalent to a matrix multiplication with a doubly-block Toeplitz matrix~\cite{jain1989fundamentals}.
However, in practice, the signal is most of the time 3-dimensional (RGB images for instance).
We call the channels of a signal \emph{channels in} denoted $cin$.
The input signal is then of size $cin \times n \times n$.
Furthermore, we perform multiple convolutions of the same signal which corresponds to the number of channels the output will have after the operation.
We call the channels of the output \emph{channels out} denoted $cout$.
Therefore, the kernel, which must take into account \emph{channels in} and \emph{channels out}, is defined as a 4-dimensional tensor of size: $\cout \times \cin \times s \times s$. 

The operation performed by a 4-dimensional kernel on a 3d signal can be expressed by the concatenation (horizontally and vertically) of doubly-block Toeplitz matrices.
Hereafter, we bound the singular value of multiple vertically stacked doubly-block Toeplitz matrices which corresponds to the operation performed by a 3d kernel on a 3d signal.

\begin{theorem}[Bound on the maximal singular value of stacked Doubly-block Toeplitz matrices] \label{theorem:theorem4} 
  Consider doubly-block Toeplitz matrices $\mathbf{D}(f_1), \dots, \mathbf{D}(f_{\cin})$ where $f_i: \mathbb{R}^2 \rightarrow \mathbb{C}$ is a generating function.
  Construct a matrix $\mathbf{M}$ with $\cin\times n^2$ rows and $n^2$ columns, as follows:
  \begin{equation}
    \mathbf{M} \triangleq \leftmat \mathbf{D}^\top(f_1), \dots, \mathbf{D}^\top(f_{\cin}) \rightmat^\top .
  \end{equation}
  Then, with $f_i$ a multivariate polynomial of the same form as Equation~\ref{equation:muli_variate_poly_on_M}, we have:
  \begin{equation} \label{equation:bound_asymptotic_equiv}
    \sigma_1\left(\mathbf{M} \right) \leq \sup_{\omega_1, \omega_2 \in \left[0, 2\pi\right]^2} \sqrt{ \sum_{i=1}^{\cin} \left|f_i\right (\omega_1, \omega_2)|^2} .
  \end{equation}
\end{theorem}

In order to prove Theorem~\ref{theorem:theorem4}, we have generalized the famous Widom identity~\cite{widom1976asymptotic} expressing the relation between Toeplitz and Hankel matrices to doubly-block Toeplitz matrices.

To have a bound on the full convolution operation, we extend Theorem~\ref{theorem:theorem4} to take into account the number of output channels.
The matrix of a full convolution operation is a block matrix where each block is a doubly-block Toeplitz matrices. Therefore, we will need the following lemma which bound the singular values of a matrix constructed from the concatenation of multiple matrix.

\begin{lemma} \label{theorem:bound_concatenation_matrices}
  Let us define matrices $\mathbf{A}_1, \dots, \mathbf{A}_p$ with $\mathbf{A}_i \in \mathbb{R}^{n \times n}$. Let us construct the matrix $\mathbf{M} \in \mathbb{R}^{n \times pn}$ as follows:
  \begin{equation}
    \mathbf{M} \triangleq \leftmat \mathbf{A}_1, \dots, \mathbf{A}_p \rightmat
  \end{equation}
  where $\leftmat\ \cdot\ \rightmat$ define the concatenation operation. Then, we can bound the singular values of the matrix $\mathbf{M}$ as follows:
  \begin{equation}
    \sigma_1(\mathbf{M}) \leq \sqrt{\sum_{i=1}^p \sigma_1(\mathbf{A}_i)^2}
  \end{equation}
\end{lemma}

\begin{figure*}[ht]
  \centering
  \begin{minipage}{.24\linewidth}
    \centering
    \includegraphics[scale=0.23]{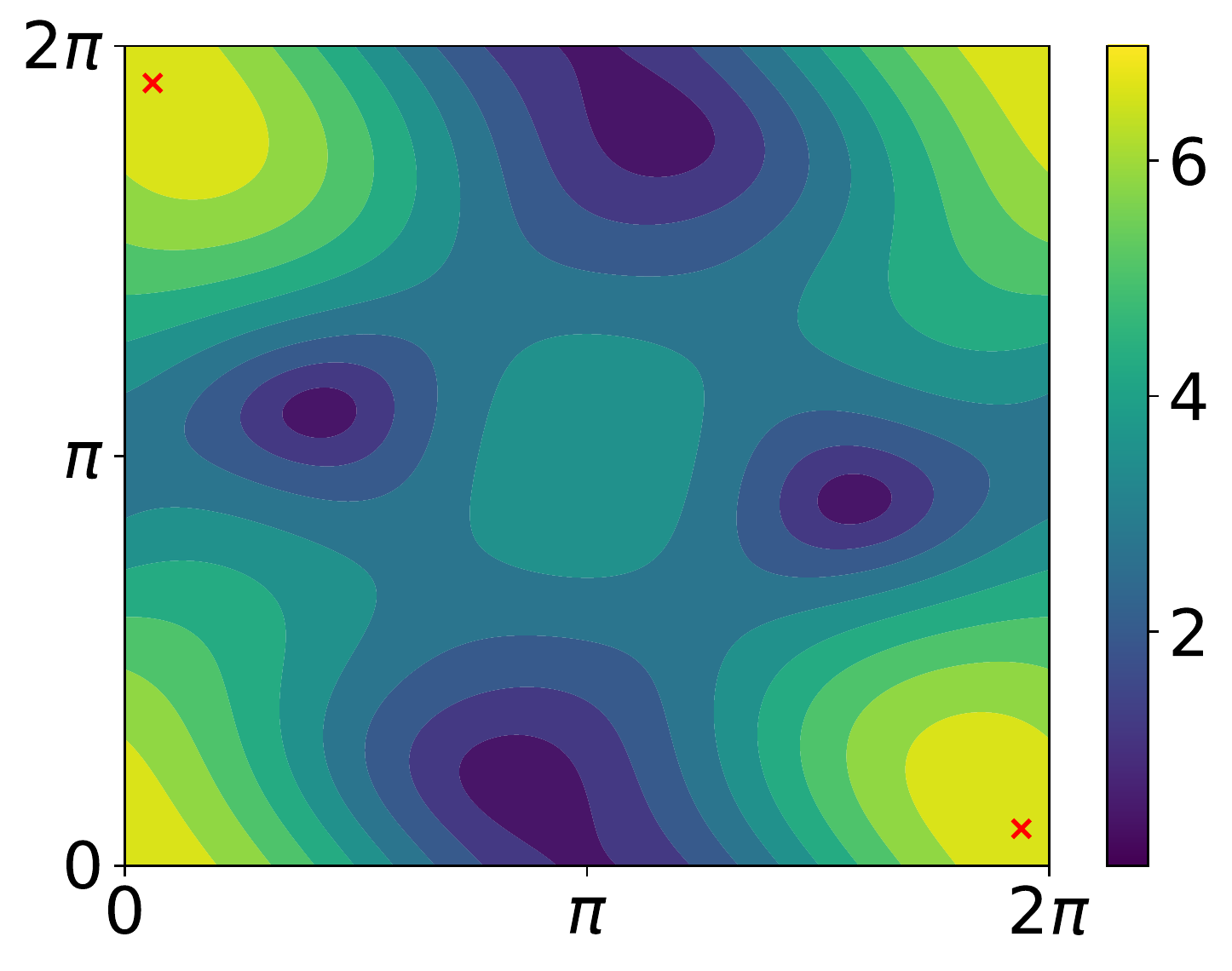}\\kernel $1\times3\times3$
  \end{minipage}
  \begin{minipage}{.24\linewidth}
      \centering
      \includegraphics[scale=0.23]{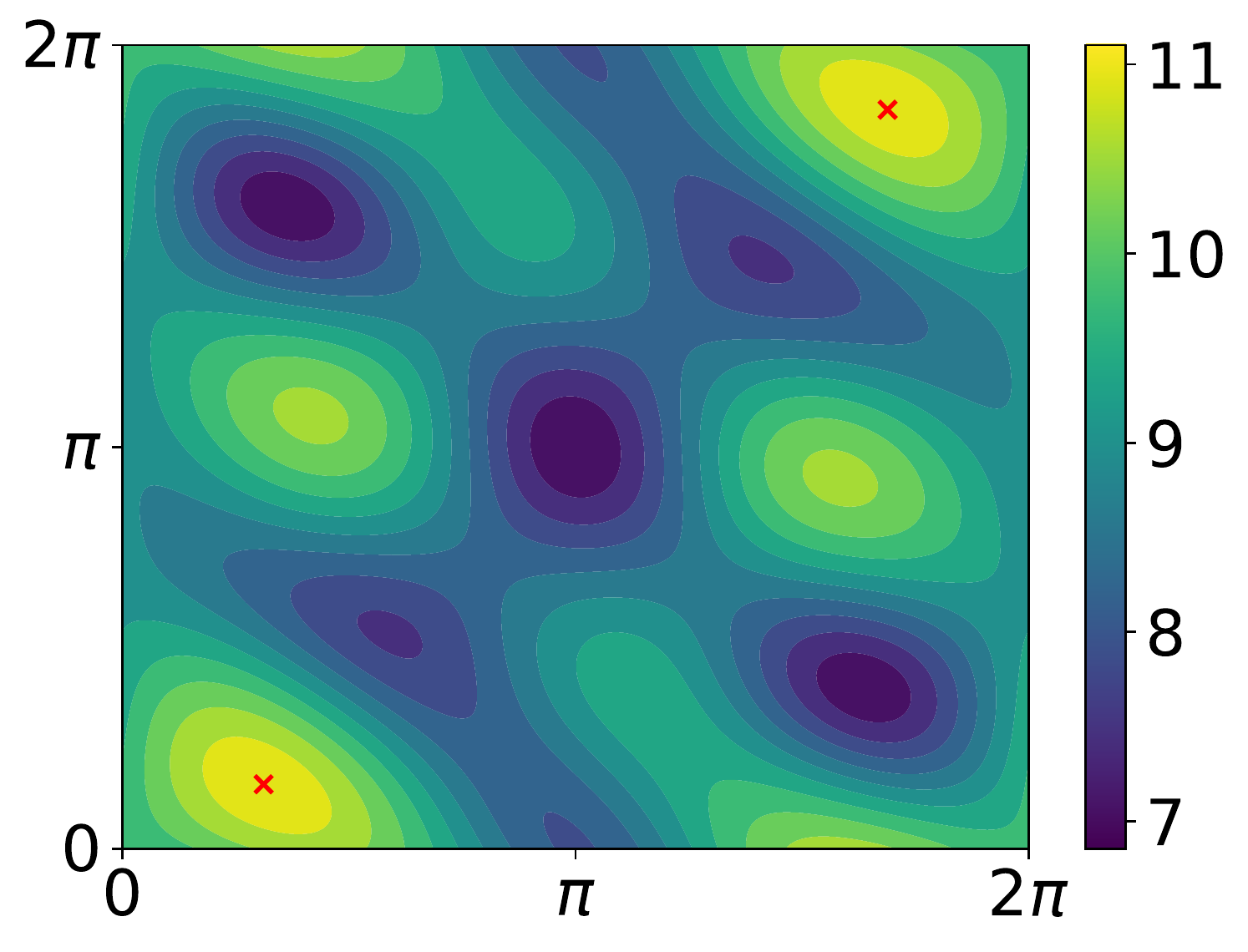}\\kernel $9\times3\times3$
  \end{minipage}
  \begin{minipage}{.24\linewidth}
      \centering
      \includegraphics[scale=0.23]{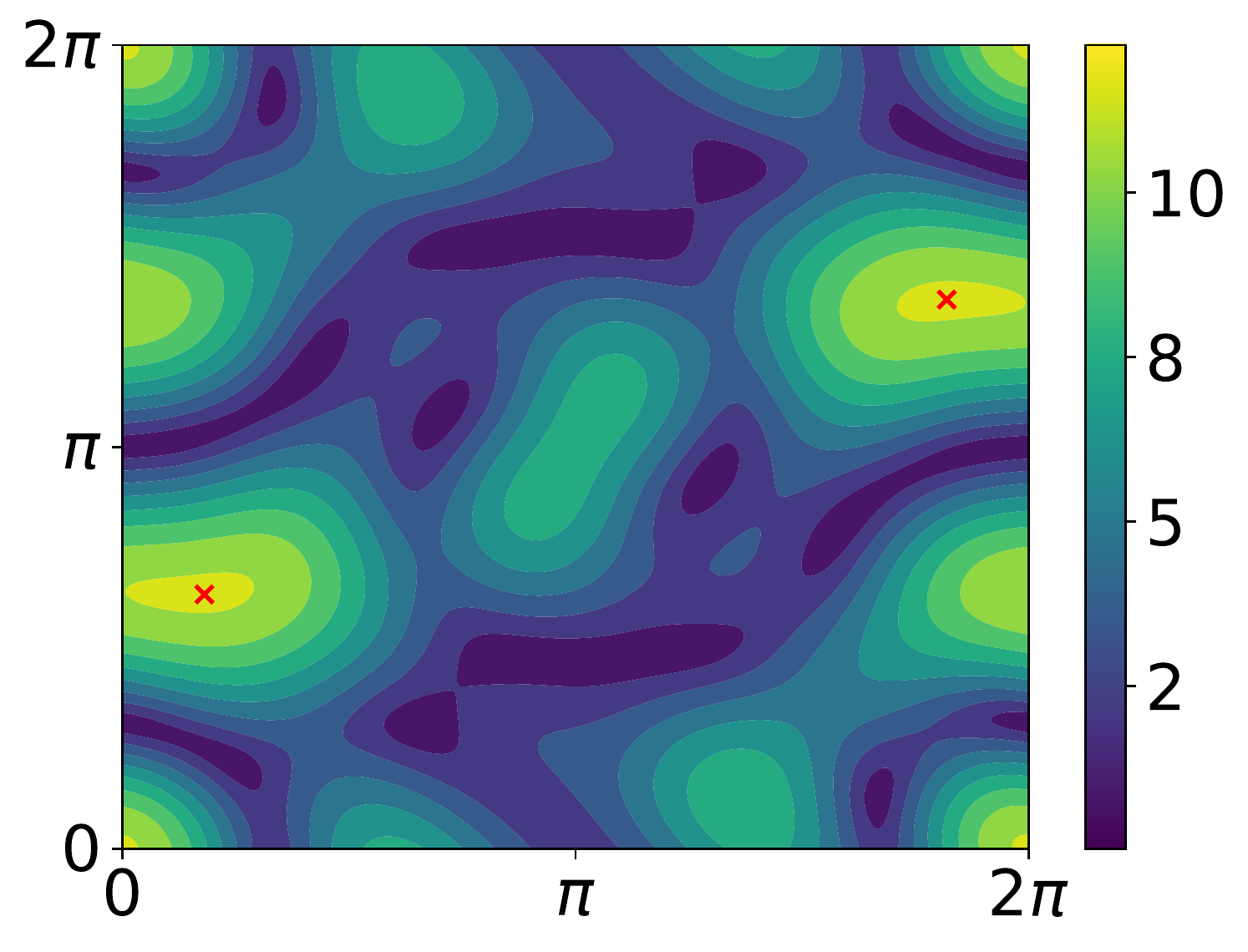}\\kernel $1\times5\times5$
  \end{minipage}
  \begin{minipage}{.24\linewidth}
      \centering
      \includegraphics[scale=0.23]{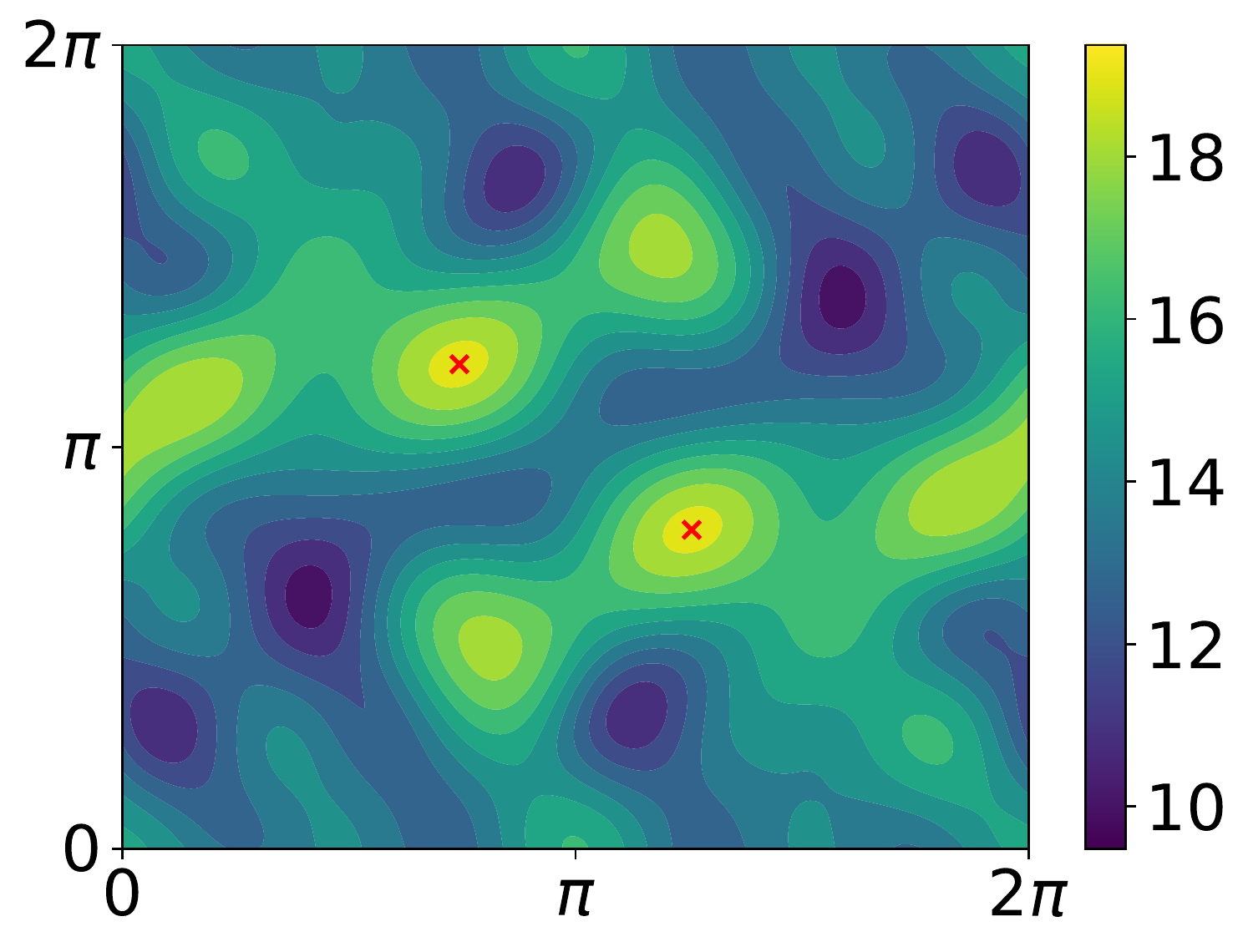}\\kernel $9\times5\times5$
  \end{minipage}
  \caption{These figures represent the contour plot of multivariate trigonometric polynomials where the values of the coefficient are the values of a random convolutional kernel. The red dots in the figures represent the maximum modulus of the trigonometric polynomials.}
  \label{figure:contour_plot_trigonometric_polynomials}
\end{figure*}%

Below, we present our main result:

\begin{theorem}[\textbf{Main Result}: Bound on the maximal singular value on the convolution operation] \label{theorem:theorem5} 
  Let us define doubly-block Toeplitz matrices $\mathbf{D}(f_{11}), \dots, \mathbf{D}(f_{\cin\times \cout})$ where $f_{ij}: \mathbb{R}^2 \rightarrow \mathbb{C}$ is a generating function. Construct a matrix $\mathbf{M}$ with $\cin\times n^2$ rows and $\cout\times n^2$ columns such as
  {\small
  \begin{equation}
    \mathbf{M} \triangleq  \leftmat\begin{array}{ccc}
    \mathbf{D}(f_{11}) & \cdots & \mathbf{D}(f_{1,\cout})   \\
    \vdots & & \vdots   \\
    \mathbf{D}(f_{\cin,1}) & \cdots & \mathbf{D}(f_{\cin,\cout}) \\
    \end{array}\rightmat .
  \end{equation}
  }
  Then, with $f_{ij}$ a multivariate polynomial of the same form as Equation~\ref{equation:muli_variate_poly_on_M}, we have:
  \begin{equation}
    \sigma_1(\mathbf{M}) \leq \sqrt{ \sum_{i=1}^{\cout} \sup_{\omega_1, \omega_2 \in [0, 2\pi]^2} \sum_{j = 1}^{\cin} \left|f_{ij}(\omega_1, \omega_2) \right|^2 } .
  \end{equation}\label{equation:lipbound_sv}
\end{theorem}
We can easily express the bound in Theorem~\ref{theorem:theorem5} with the values of a 4-dimensional kernel.
Let us define a kernel $\mathbf{k} \in \mathbb{R}^{\cout \times \cin \times s \times s}$, a padding $p \in \mathbb{N}$ and $d = \lfloor s / 2 \rfloor$ the degree of the trigonometric polynomial, then:
\begin{equation}
  f_{ij}(\omega_1, \omega_2) = \sum_{h_1 = -d}^d \sum_{h_2 = -d}^d k_{i, j, h_1,h_2} e^{\ci (h_1 \omega_1 + h_2 \omega_2)}.
\end{equation}
where $k_{i, j, h_1,h_2} = \leftmat \mathbf{k} \rightmat_{i, j, a, b}$ with $a =  s - p - 1 + i$ and $b =  s - p - 1 + j$.

In the rest of the paper, we will refer to the bound in Theorem~\ref{theorem:theorem5} applied to a kernel as $\LipBound$ and we denote $\LipBound(\mathbf{k})$ the Lipschitz upper bound of the convolution performed by the kernel $\mathbf{k}$.

\section{Computation and Performance Analysis of LipBound}
\label{section:computation_performance_lipbound}

This section aims at analyzing the bound on the singular values introduced in Theorem~\ref{theorem:theorem5}.
First, we present an algorithm to efficiently compute the bound, we analyze its tightness by comparing it against the true maximal singular value.
Finally, we compare the efficiency and the accuracy of our bound against the state-of-the-art. 

\subsection{Computing the maximum modulus of a trigonometric polynomial}
\label{subsection:computing_max_modulus_trig_polynomial}

In order to compute $\LipBound$ from Theorem~\ref{theorem:theorem5}, we have to compute the maximum modulus of several trigonometric polynomials.
However, finding the maximum modulus of a trigonometric polynomial has been known to be NP-hard~\cite{pfister2018bounding}, and in practice they exhibit low convexity (see Figure~\ref{figure:contour_plot_trigonometric_polynomials}).
We found that for 2-dimensional kernels, a simple grid search algorithm such as PolyGrid (see Algorithm~\ref{algo:PolyGrid}), works better than more sophisticated approximation algorithms (\eg  ~\citealt{green1999calculating,de2009finding}).
This is because the complexity of the computation depends on the degree of the polynomial which is equal to $\lfloor s / 2 \rfloor$ where $s$ is the size of the kernel and is usually small in most practical settings (\eg $s=3$).
Furthermore, the grid search algorithm can be parallelized effectively on CPUs or GPUs and runs within less time than alternatives with lower asymptotic complexity.

To fix the number of samples $S$ in the grid search, we rely on the work of~\citet{pfister2018bounding}, who has analyzed the quality of the approximation depending on $S$.
Following, this work we first define $\Theta_S$, the set of $S$ equidistant sampling points as follows:
\begin{equation}
  \Theta_S \triangleq \left\{ \omega \mid \omega = k \cdot \frac{2\pi}{S} \mbox{ with }  k = 0,\ldots, S-1 \right\}.
\end{equation}
Then, for $f: [0, 2\pi]^2 \rightarrow \mathbb{C}$, we have:
\begin{equation}
  \max_{\omega_1, \omega_2 \in [0,2\pi]^2} \left| f(\omega_1, \omega_2) \right| \leq (1 - \alpha)^{-1} \max_{\omega_1', \omega_2' \in \Theta_S^2} \left| f(\omega_1', \omega_2') \right|,
\end{equation}
where $d$ is the degree of the polynomial and $\alpha = 2d / S$.
For a $3\times3$ kernel which gives a trigonometric polynomial of degree 1, we use $S = 10$ which gives $\alpha = 0.2$.
Using this result, we can now compute $\LipBound$ for a convolution operator with $cout$ output channels as per Theorem~\ref{theorem:theorem4}.

\begin{algorithm}
  \begin{algorithmic}[1]
    \State{\textbf{input} polynomial $f$, number of samples $S$}
    \State{\textbf{output} approximated maximum modulus of $f$}
    \State $\sigma \gets 0$, $\omega_1 \gets 0$, $\epsilon \leftarrow 2\pi / S$
    \For{$i=0$ \textbf{to} $S-1$}
      \State $\omega_1 \gets \omega_1 + \epsilon$, $\omega_2 \gets 0$
      \For{$j=0$ \textbf{to} $S-1$}
	\State $\omega_2 \gets \omega_2 + \epsilon$
	\State $\sigma \gets \max( \sigma, f(\omega_1, \omega_2))$
      \EndFor
    \EndFor
    \State{\textbf{return} $\sigma$}
  \end{algorithmic}
  \caption{PolyGrid}
  \label{algo:PolyGrid}
\end{algorithm}

\subsection{Analysis of the tightness of the bound}
\label{section:analysis_tightness_bound}

In this section, we study the tightness of the bound with respect to the dimensions of the doubly-block Toeplitz matrices.
For each $n \in \mathbb{N}$, we define the matrix $\mathbf{M}^{(n)}$ of size $kn^2 \times n^2$ as follows:
\begin{equation}
  \mathbf{M}^{(n)} \triangleq \textstyle \leftmat \mathbf{D}^{(n)\top}(f_1), \dots, \mathbf{D}^{(n)\top}(f_k) \textstyle \rightmat^\top
\end{equation}
where the matrices $\mathbf{D}^{(n)}(f_i)$ are of size $n^2 \times n^2$. 
To analyze the tightness of the bound, we define the function $\Gamma$, which computes the difference between $\LipBound$ and the maximal singular value of the function $\mathbf{M}^{(n)}$:
\begin{equation}
  \Gamma(n) = \LipBound(\mathbf{k}_{\mathbf{M}^{(n)}}) - \sigma_1(\mathbf{M}^{(n)})
  \label{equation:function_convergence}
\end{equation}
where $\mathbf{k}_{\mathbf{M}^{(n)}}$ is the convolution kernel of the convolution defined by the matrix $\mathbf{M}^{(n)}$.

To compute the exact largest singular value of $\mathbf{M}^{(n)}$ for a specific $n$, we use the Implicitly Restarted Arnoldi Method (IRAM) ~\cite{lehoucq1996deflation} available in SciPy.
The results of this experiment are presented in Figure~\ref{figure:convergence_bound}.
We observe that the difference between the bound and the actual value (approximation gap) quickly decreases as the input size increases.
For an input size of $50$, the approximation gap is as low as $0.012$ using a standard $6\times3\times3$ convolution kernel.
For a larger input size such as ImageNet ($224$), the gap is lower than $4.10^{-4}$.
Therefore $\LipBound$ gives an almost exact value of the maximal singular value of the operator matrix for most realistic settings.

\begin{table*}[ht]
  \centering
  \caption{The following table compares different approaches for computing an approximation of the maximal singular value of a convolutional layer. It shows the ratio between the approximation and the true maximal singular value. The approximation is better for a ratio close to one.}
    \begin{tabular}{lrccrcc}
    \toprule
      &   & \multicolumn{2}{c}{\textbf{1x3x3}} &   & \multicolumn{2}{c}{\textbf{32x3x3}} \\
    \cmidrule{3-4}\cmidrule{6-7}  &   & \textbf{Ratio} & \textbf{Time (ms)} &   & \textbf{Ratio} & \textbf{Time (ms)} \\
    \midrule
    \citeauthor{sedghi2018the} &   & $\phantom{.}0.431\pm0.042$ & $1088\pm251$ &   & $\phantom{.}0.666\pm0.123$ & $1729\pm399$ \\
    \citeauthor{singla2019bounding} &   & $\phantom{.}1.293\pm0.126$ & $\phantom{..}1.90\pm0.48$ &   & $\phantom{.}1.441\pm0.188$ & $\phantom{..}1.90\pm0.46$ \\
    \citeauthor{farnia2018generalizable} (10 iter) &   & $\phantom{.}0.973\pm0.006$ & $\phantom{..}4.30\pm0.64$ &   & $\phantom{.}0.972\pm0.004$ & $\phantom{..}4.93\pm0.67$ \\
    \midrule
    \textbf{LipBound (Ours)} &   & $\mathbf{0.992}\pm0.012$ & $\phantom{.}\mathbf{0.49}\pm0.05$ &   & $\mathbf{0.984}\pm0.021$ & $\phantom{.}\mathbf{0.63}\pm0.46$ \\
    \bottomrule
    \end{tabular}%
  \label{tab:comparaison}%
\end{table*}

\subsection{Comparison of LipBound with other state-of-the-art approaches}
\label{subsection:comparaison_sota}

In this section we compare our PolyGrid algorithm with the values obtained using alternative approaches.
We consider the 3 alternative techniques by ~\citet{sedghi2018the}, by~\citet{singla2019bounding} and by ~\citet{farnia2018generalizable} which have been described in Section~\ref{section:related_work}.

To compare the different approaches, we extracted 20 kernels from a trained model.
For each kernel we construct the corresponding doubly-block Toeplitz matrix and compute its largest singular value.
Then, we compute the ratio between the approximation obtained with the approach in consideration and the exact singular value obtained by SVD, and average the ratios over the 20 kernels.
Thus good approximations result in approximation ratios that are close to 1.
The results of this experiment are presented in Table~\ref{tab:comparaison}.
The comparison has been made on a Tesla V100 GPU. 
The time was computed with the PyTorch CUDA profiler and we warmed up the GPU before starting the timer.

\begin{figure}[ht]
  \centering
  \includegraphics[scale=0.36]{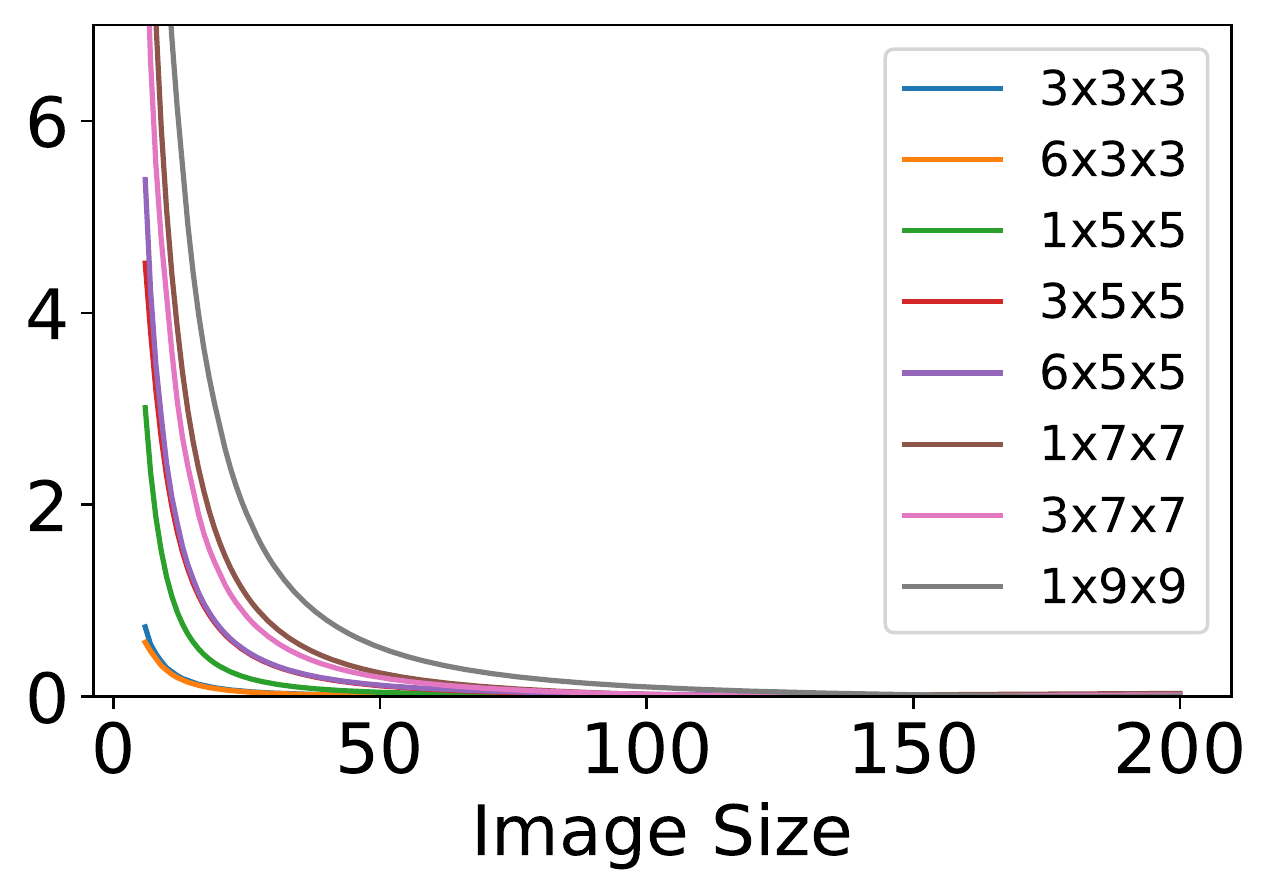}
  \caption{This graph represents the function $\Gamma(n)$ defined in Section~\ref{section:analysis_tightness_bound} for different kernel size.}
  \label{figure:convergence_bound}
\end{figure}%

The method introduced by~\citet{sedghi2018the} computes an  approximation of the singular values of convolutional layers.
We can see in Table~\ref{tab:comparaison} that the value is off by an important margin.
This technique is also computationally expensive as it requires computing the SVD of $n^2$ small matrices where $n$ is the size of inputs.
\citet{singla2019bounding} have shown that the singular value of the reshape kernel is a bound on the maximal singular value of the convolution layer.
Their approach is very efficient but the approximation is loose and overestimate the real value.
As said previously, the power method provides a good approximation at the expense of the efficiency.
We use the special Convolutional Power Method from ~\cite{farnia2018generalizable} with 10 iterations.
The results show that our proposed technique: PolyGrid algorithm can get the best of both worlds.
It achieves a near perfect accuracy while being very efficient to compute. 

We provide in the supplementary material a benchmark on the efficiency of $\LipBound$ on multiple convolutional architectures. 

\section{Application: Lipschitz Regularization for Adversarial Robustness}
\label{section:experiments}

\begin{figure*}[ht]
  \centering
  \begin{minipage}{.24\linewidth}
    \centering
    \includegraphics[scale=0.22]{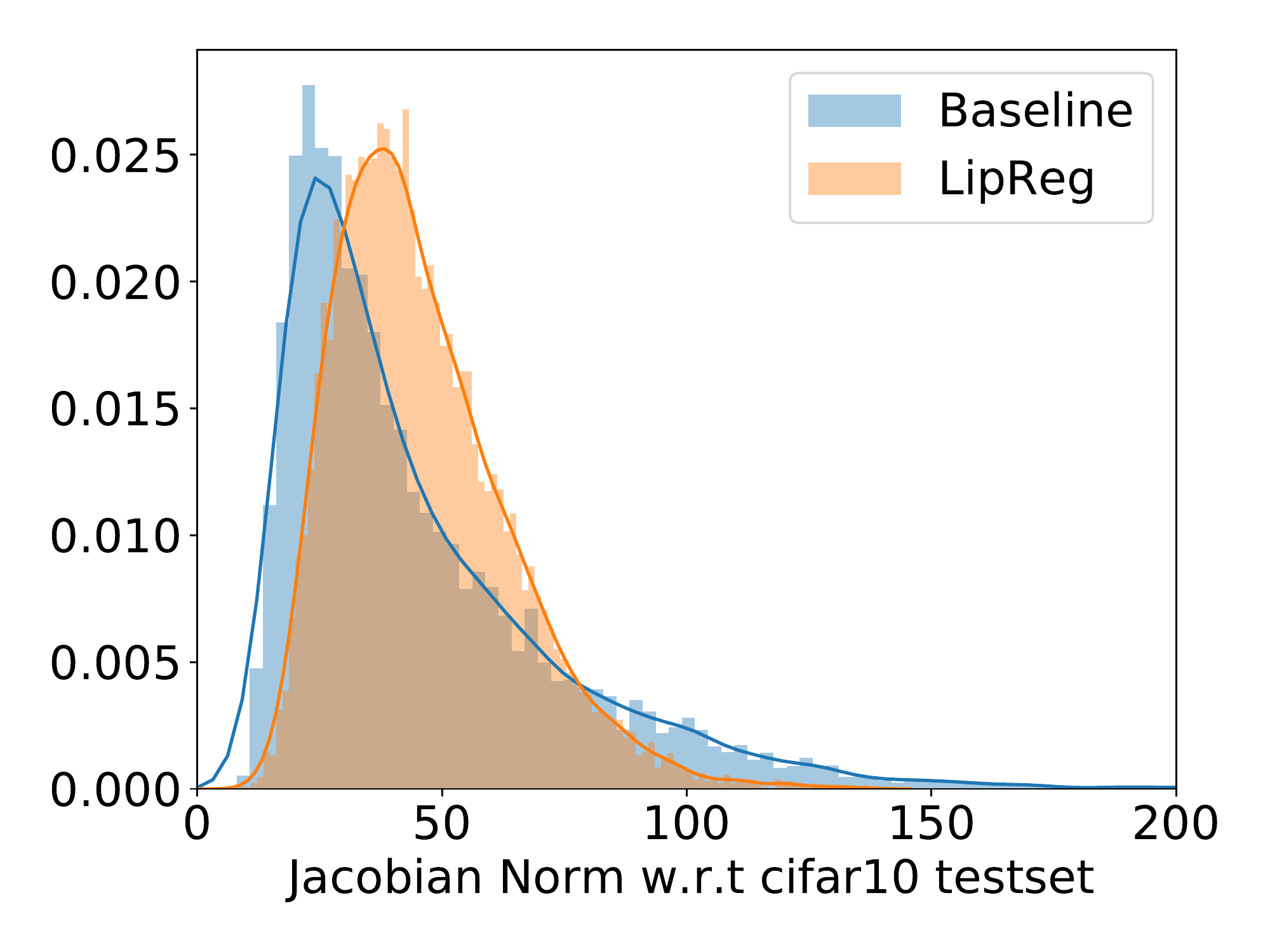}\\{(a)}
  \end{minipage}
  \hfill
  \begin{minipage}{.24\linewidth}
    \centering
    \includegraphics[scale=0.22]{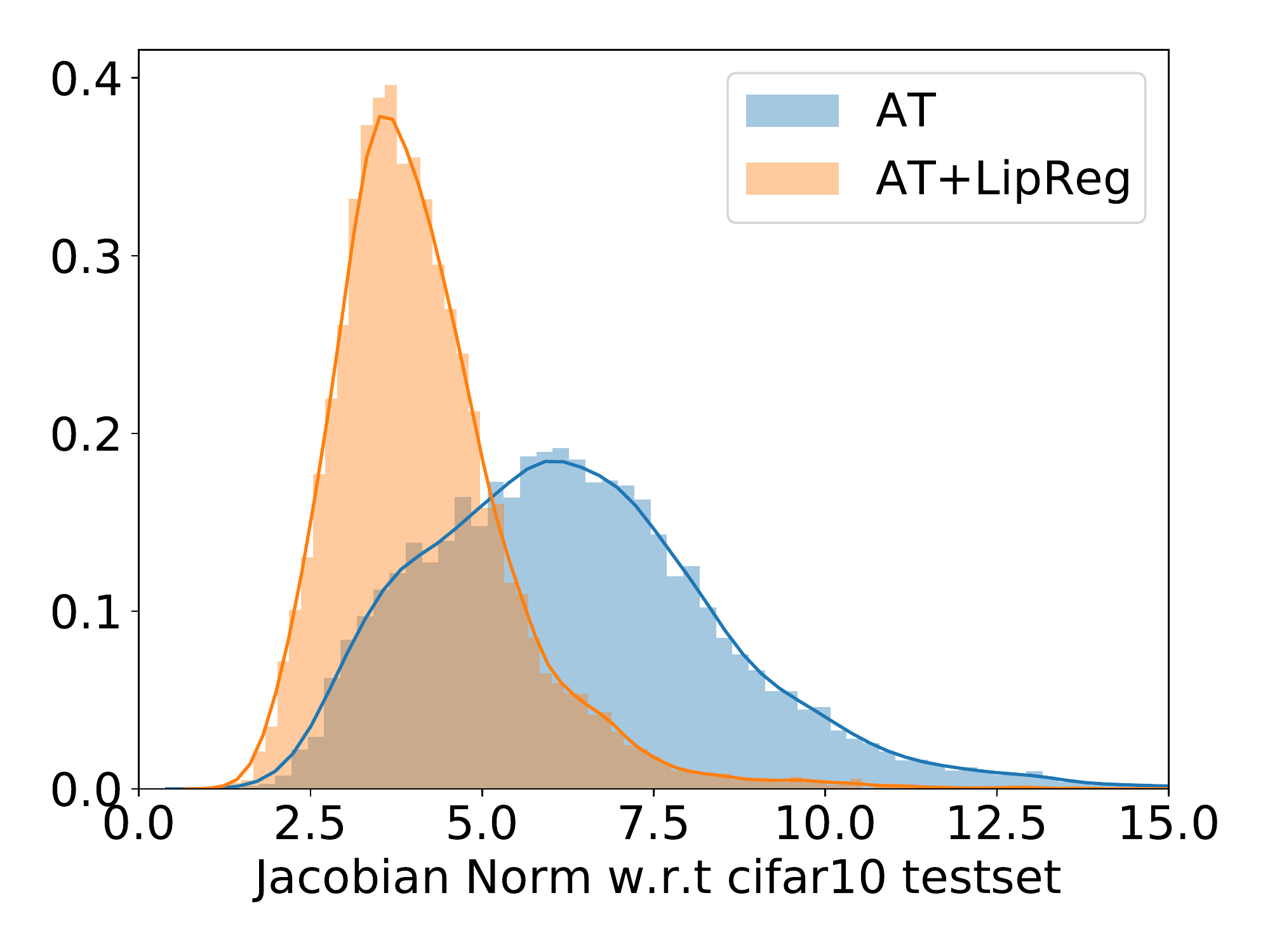}\\{(b)}
  \end{minipage}
  \hfill
  \begin{minipage}{.24\linewidth}
    \centering
    \includegraphics[scale=0.22]{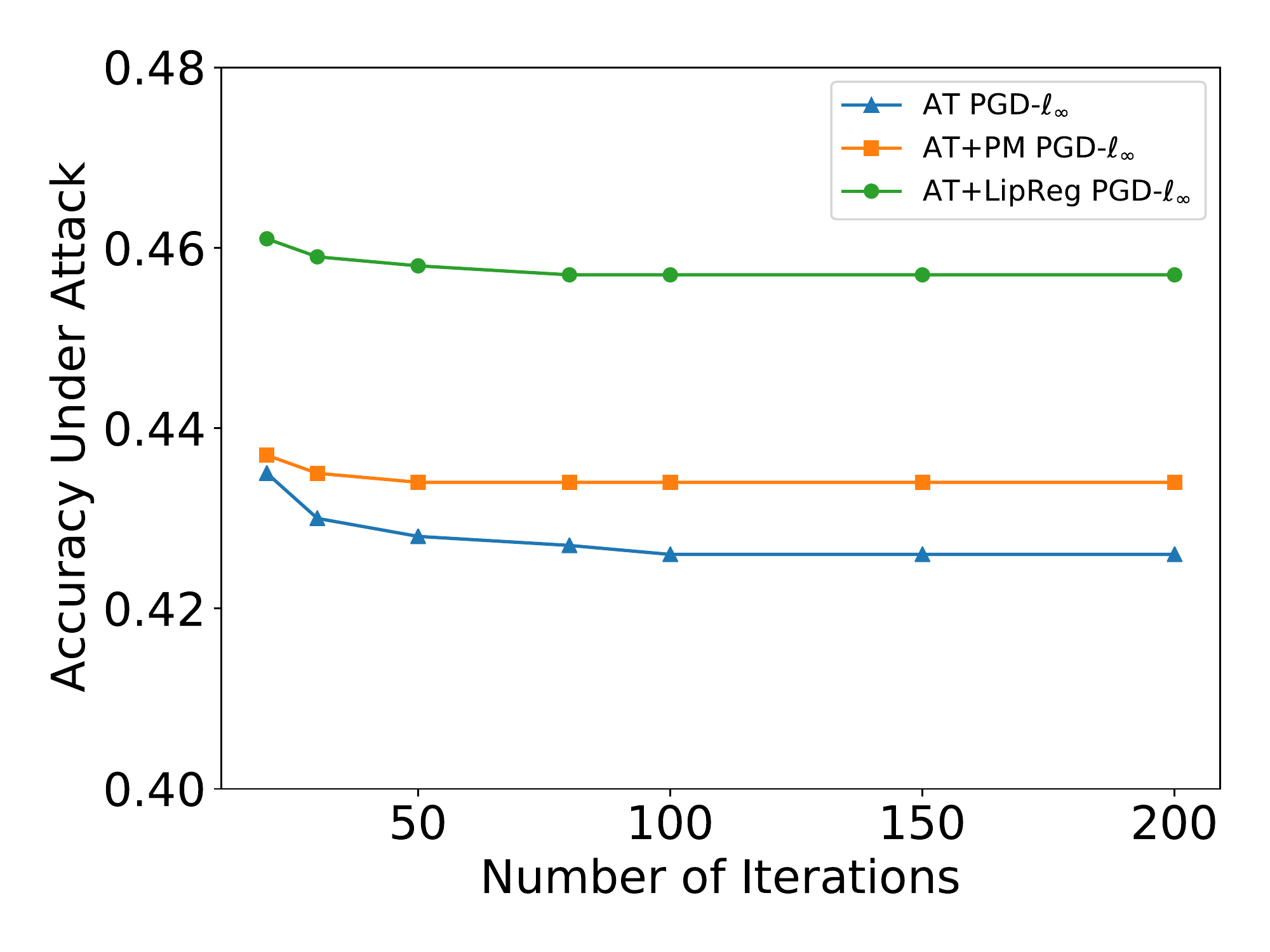}\\{(c)}
  \end{minipage}
  \hfill
  \begin{minipage}{.24\linewidth}
    \centering
    \includegraphics[scale=0.22]{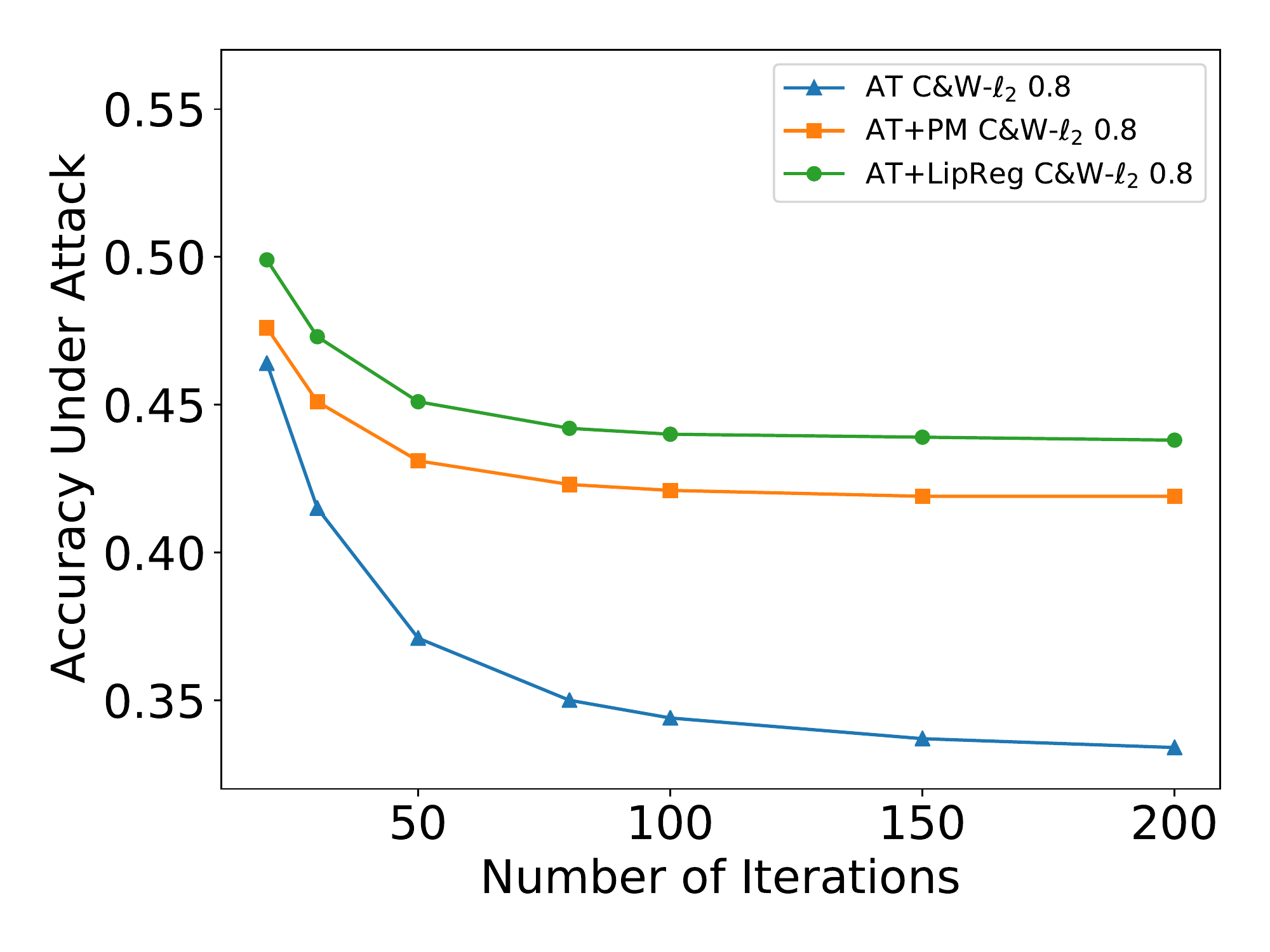}\\{(d)}
  \end{minipage}
  \caption{Figures (a) and (b) show the distribution of the norm of the Jacobian matrix w.r.t the CIFAR10 test set from a Wide Resnet trained with different schemes. Although Lipschitz regularization is not a Jacobian regularization, we can observe a clear shift in the distribution. This suggests that our method does not only work layer-wise, but also at the level of the entire network. Figures (c) and (d) show the Accuracy under attack on CIFAR10 test set with PGD-$\ell_\infty$ and C\&W-$\ell_2$ attacks for several classifiers trained with Adversarial Training given the number of iterations.}
  \label{figure:dist_jacobian_attacks_iter}
\end{figure*}%

\begin{table*}[t]
  \centering
  \caption{This table shows the Accuracy under $\ell_2$ and $\ell_\infty$ attacks of CIFAR10/100 datasets. We compare vanilla Adversarial Training with the combination of Lipschitz regularization and Adversarial Training. We also compare the effectiveness of the power method by~\citet{farnia2018generalizable} and $\LipBound$. The parameters $\lambda_2$ (Eq.~\ref{equation:obj_function}) is equal to $0.008$ for AT+PM and AT+LipReg. It has been chosen from a grid search among 10 values. The attacks below are computed with 200 iterations. }
    \begin{tabular}{llcccc}
    \toprule
    \textbf{Dataset} & \textbf{Model} & \textbf{Accuracy} & \textbf{PGD-$\ell_\infty$} & \textbf{C\&W-$\ell_2$ 0.6} & \textbf{C\&W-$\ell_2$ 0.8} \\
    \midrule
    \multirow{4}[0]{*}{CIFAR10} & \textbf{Baseline} & $\mathbf{0.953}\pm0.001$ & $\phantom{.}0.000\pm0.000$ & $\phantom{.}0.002\pm0.000$ & $\phantom{.}0.000\pm0.000$ \\
    & \textbf{AT} & $\phantom{.}0.864\pm0.001$ & $\phantom{.}0.426\pm0.000$ & $\phantom{.}0.477\pm0.000$ & $\phantom{.}0.334\pm0.000$ \\
    & \textbf{AT+PM} & $\phantom{.}0.788\pm0.010$ & $\phantom{.}0.434\pm0.007$ & $\phantom{.}0.521\pm0.005$	 & $\phantom{.}0.419\pm0.003$ \\
    & \textbf{AT+LipReg} & $\phantom{.}0.808\pm0.022$ & $\mathbf{0.457}\pm0.002$ & $\mathbf{0.547}\pm0.022$ & $\mathbf{0.438}\pm0.020$ \\
    \midrule
    \multirow{3}[0]{*}{CIFAR100} & \textbf{Baseline} & $\mathbf{0.792}\pm0.000$ & $\phantom{.}0.000\pm0.000$ & $\phantom{.}0.001\pm0.000$ & $\phantom{.}0.000\pm0.000$ \\
    & \textbf{AT} & \phantom{.}$0.591\pm0.000$ & $\phantom{.}0.199\pm0.000$ & $\phantom{.}0.263\pm0.000$ & $\phantom{.}0.183\pm0.000$ \\
    & \textbf{AT+LipReg} & \phantom{.}$0.552\pm0.019$ & $\mathbf{0.215}\pm0.004$ & $\mathbf{0.294}\pm0.010$ & $\mathbf{0.226}\pm0.008$ \\
    \bottomrule
    \end{tabular}%
  \label{table:cifar_robustness}%
\end{table*}%

\begin{table*}[htbp]
  \centering
  \caption{This table shows the accuracy and accuracy under $\ell_2$ and $\ell_\infty$ attack of ImageNet dataset. We compare Adversarial Training with the combination of Lipschitz regularization and Adversarial Training \cite{madry2018towards}. }
    \begin{tabular}{llccccccccc}
    \toprule
    \multicolumn{1}{c}{\multirow{2}[4]{*}{\textbf{Dataset}}} & \multicolumn{1}{c}{\multirow{2}[4]{*}{\textbf{Model}}} & \multicolumn{1}{c}{\multirow{2}[4]{*}{\textbf{LipReg} $\lambda_2$}} & \multicolumn{1}{c}{\multirow{2}[4]{*}{\textbf{Natural}}} &   & \multicolumn{2}{c}{\textbf{PGD-}$\ell_\infty$} &   & \multicolumn{3}{c}{\textbf{C\&W-}$\ell_2$} \\
    \cmidrule{6-7}\cmidrule{9-11}      &  &  &  &  & \multicolumn{1}{c}{0.02} & \multicolumn{1}{c}{0.031} &   & \multicolumn{1}{c}{1.00} & \multicolumn{1}{c}{2.00} & \multicolumn{1}{c}{3.00} \\
    \midrule
    \multirow{4}[0]{*}{ImageNet} & \textbf{Baseline} \cite{he2016deep} & -- & \textbf{0.782} & & 0.000 & 0.000 & & 0.000 & 0.000 & 0.000 \\ 
    & \textbf{AT} & \multicolumn{1}{c}{--} & 0.509 &   & 0.251 & 0.118 &   & 0.307 & 0.168 & 0.099 \\
    & \textbf{AT+LipReg} & 0.0006 & \textbf{0.515} &   & \textbf{0.255} & \textbf{0.121} &   & \textbf{0.316} & \textbf{0.177} & \textbf{0.105} \\
    & \textbf{AT+LipReg} & 0.0010 & \textbf{0.519} &   & \textbf{0.259} & \textbf{0.123} &   & \textbf{0.338} & \textbf{0.204} & \textbf{0.129} \\
    \bottomrule
    \end{tabular}%
    \label{table:imagenet_robustness}
\end{table*}%

One promising application of Lipschitz regularization is in the area of adversarial robustness. 
Empirical techniques to improve robustness against adversarial examples such as Adversarial Training only impact the training data,  and often show poor generalization capabilities~\cite{schmidt2018adversarially}. 
\citet{farnia2018generalizable} have shown that the adversarial generalization error depends on the Lipschitz constant of the network, which suggests that the adversarial test error can be improved by applying Lipschitz regularization in addition to adversarial training. 

In this section, we illustrate the usefulness of LipBound by training a state-of-the-art Wide ResNet architecture~\citep{zagoruyko2016wide} with Lipschitz regularization and adversarial training.
Our regularization scheme is inspired by the one used by \citet{yoshida2017spectral} but instead of using the power method, we use our \textbf{PloyGrid} algorithm presented in Section~\ref{subsection:computing_max_modulus_trig_polynomial} which efficiently computes an upper bound on the maximal singular value of convolutional layers. 

We introduce the \textbf{AT+LipReg} loss to combine Adversarial Training and our Lipschitz regularization scheme in which layers with a large Lipschitz constant are penalized.
We consider a neural network $\mathcal N_\theta : \mathcal X \rightarrow \mathcal Y$ with $\ell$ layers $\phi^{(1)}_{\theta_1}, \ldots, \phi^{(\ell)}_{\theta_\ell}$ where $\theta^{(1)}, \ldots, \theta^{(\ell -1)}$ are the kernels of the first $\ell - 1$ convolutional layers and $\theta_\ell$ is the weight matrix of the last fully-connected  layer $\phi^{(\ell)}_{\theta_\ell}$.
Given a distribution $\mathcal D$ over $\mathcal X \times \mathcal Y$, we can train the parameters $\theta$ of the network by minimizing the AT+LipReg loss as follows:
\begin{align} 
  &\min_{\theta} \mathbb E_{x,y \sim \mathcal D} \left[ \max_{\norm{\tau}_\infty \leq \epsilon}  \mathcal{L}(\mathcal{N}_\theta(x + \tau), y)  \right. \notag \\
    &\quad \left.  + \lambda_1 \sum_{i=1}^\ell {\textstyle \norm{\theta_i}_\text{F}} + \lambda_2 \sum_{i=1}^{\ell-1} \log\left( \LipBound\left(\theta_i\right)\right) \right]
  \label{equation:obj_function}
\end{align}
where $\mathcal L$ is the cross-entropy loss function, and $\lambda_1$, $\lambda_2$ are two user-defined hyper-parameters.
Note that regularizing the sum of logs is equivalent to regularizing the product of all the $\LipBound$ which is an upper bound on the global Lipschitz constant.
In practice, we also include the upper bound on the Lipschitz of the batch normalization because we can compute it very efficiently (see C.4.1 of ~\citealt{tsuzuku2018lipschitz}) but we omit the last fully connected layer.

In this section, we compare the robustness of Adversarial Training~\cite{goodfellow2014explaining, madry2018towards} against the combination of Adversarial Training and Lipschitz regularization.
To regularize the Lipschitz constant of the network, we use the objective function defined in Equation~\ref{equation:obj_function}.
We train Lipschitz regularized neural networks with LipBound (Theorem~\ref{theorem:theorem5}) implemented with PolyGrid (Algorithm~\ref{algo:PolyGrid}) (AT+LipBound) with $S = 10$ or with the specific power method for convolutions introduced by~\citet{farnia2018generalizable} with 10 iterations (AT+PM). 

Table~\ref{table:cifar_robustness} shows the gain in robustness against strong adversarial attacks across different datasets.
We can observe that both AT+LipBound and AT+PM offer a better defense against adversarial attacks and that AT+LipBound offers a further improvement over the Power Method.
The Figure~\ref{figure:dist_jacobian_attacks_iter} (c) and (d) shows the Accuracy under attack with different number of iterations. 
Table~\ref{table:imagenet_robustness} presents our results on the ImageNet Dataset.
First, we can observe that the networks AT+LipReg offers a better generalization than with standalone Adversarial Training.
Secondly, we can observe the gain in robustness against strong adversarial attacks. Network trained with Lipschitz regularization and Adversarial Training offer a consistent increase in robustness across $\ell_\infty$ and $\ell_2$ attacks with different $\epsilon$ value. We can also note that increasing the regularization lead to an increase in generalization and robustness.  

Finally, we also conducted an experiment to study the impact of the regularization on the gradients of the whole network by measuring the norm of the Jacobian matrix, averaged over the inputs from the test set.
The results of this experiment are presented in Figure~\ref{figure:dist_jacobian_attacks_iter}(a) and show more concentrated gradients with  Lipschitz regularization, which is the expected effect.
This suggests that our method does not only work layer-wise, but also at the level of the entire network.
A second experiment, using Adversarial Training, presented in Figure~\ref{figure:dist_jacobian_attacks_iter}(b) demonstrates that the effect is even stronger when the two techniques are combined together.
This corroborates the work by~\citet{farnia2018generalizable}.
It also demonstrates that Lipschitz regularization and Adversarial Training (or other Jacobian regularization techniques) are complementary.
Hence they offer an increased robustness to adversarial attacks as demonstrated above.

\paragraph{Experimental Settings CIFAR10/100 Dataset}
For all our experiments, we use the Wide ResNet architecture introduced by~\citet{zagoruyko2016wide} to train our classifiers.
We use Wide Resnet networks with 28 layers and a width factor of 10. We train our networks for 200 epochs with a batch size of $200$. We use Stochastic Gradient Descent with a momentum of $0.9$, an initial learning rate of $0.1$ with exponential decay of 0.1 (MultiStepLR gamma = 0.1) after the epochs $60$, $120$ and $160$. 
For Adversarial Training ~\cite{madry2018towards}, we use Projected Gradient Descent with an $\epsilon = 8/255 (\approx 0.031)$, a step size of $\epsilon/5 (\approx 0.0062)$ and 10 iterations, we use a random initialization but run the attack only once. 
To evaluate the robustness of our classifiers, we rigorously followed the experimental protocol proposed by~\citet{tramer2020adaptive} and~\citet{carlini2019evaluating}.
More precisely, as an $\ell_\infty$ attack, we use PGD with the same parameters ($\epsilon = 8/255$, a step size of $\epsilon/5$) but we increase the number of iterations up to 200 with 10 restarts.
For each image, we select the perturbation that maximizes the loss among all the iterations and the 10 restarts.
As $\ell_2$ attacks, we use a bounded version of the~\citet{carlini2017towards} attack.
We choose $0.6$ and $0.8$ as bounds for the $\ell_2$ perturbation.
Note that the $\ell_2$ ball with a radius of $0.8$ has approximately the same volume as the $\ell_\infty$ ball with a radius of $0.031$ for the dimensionality of CIFAR10/100.

\paragraph{Experimental Settings for ImageNet Dataset}

For all our experiments, we use the Resnet-101 architecture \cite{he2016deep}.
We have used Stochastic Gradient Descent with a momentum of $0.9$, a weight decay of $0.0001$, label smoothing of $0.1$, an initial learning rate of $0.1$ with exponential decay of $0.1$ (MultiStepLR gamma = $0.1$) after the epochs $30$ and $60$.
We have used Exponential Moving Average over the weights with a decay of $0.999$.
We have trained our networks for 80 epochs with a batch size of $4096$.
For Adversarial Training, we have used PGD with 5 iterations, $\epsilon = 8/255 (\approx 0.031)$ and a step size of $\epsilon/5 (\approx 0.0062)$. 
To evaluate the robustness of our classifiers on ImageNet Dataset, we have used an $\ell_\infty$ and an $\ell_2$ attacks.
More precisely, as an $\ell_\infty$ attack, we use PGD with an epsilon of 0.02 and 0.031, a step size of $\epsilon/5$) with a number of iterations to 30 with 5 restarts.
For each image, we select the perturbation that maximizes the loss among all the iterations and the 10 restarts.
As $\ell_2$ attacks, we use a bounded version of the~\citet{carlini2017towards} attack.
We have used $1$, $2$ and $3$ as bounds for the $\ell_2$ perturbation.

\section{Conclusion}
\label{section:conclusion}

In this paper, we introduced a new bound on the Lipschitz constant of convolutional layers that is both accurate and efficient to compute.
We used this bound to regularize the Lipschitz constant of neural networks and demonstrated its computational efficiency in training large neural networks with a regularized Lipschitz constant.
As an illustrative example, we combined our bound with adversarial training, and showed that this increases the robustness of the trained networks to  adversarial attacks.
The scope of our results goes beyond this application and can be used in a wide variety of settings, for example, to stabilize the training of Generative Adversarial Networks (GANs) and invertible networks, or to improve generalization capabilities of classifiers.
Our future work will focus on investigating these fields. 

\section*{Acknowledgements}
We would like to thank Rafael Pinot and Geovani Rizk for their valuable insights. This work was granted access to the HPC resources of IDRIS under the allocation 2020-101141 made by GENCI.

%% file: matsup.tex
\appendix
\part*{Supplementary Material}

\sectionapp{Notations}

Below are the notations we will use for the theorems and proofs. 
\begin{itemize}
\itemsep0em 
  \item Let $\ci = \sqrt{-1}$.
  \item We denote $\sigma_1(\mathbf{A}) = \sigma_{max}(\mathbf{A})$ the maximum singular value of the matrix $\mathbf{A}$. 
  \item We denote $\lambda_1(\mathbf{A}) = \lambda_{max}(\mathbf{A})$ the maximum eigenvalue of the Hermitian matrix $\mathbf{A}$. 
  \item For any function $f: \mathcal{X} \rightarrow \mathbb{C}$, we denote $f^*$ the conjugate function of $f$.
  \item Let $\mathbf{A}$ be a $n \times n$ symmetric real matrix, we say that 
  \begin{itemize}
    \item[] $\mathbf{A}$ is positive definite, and we note $\mathbf{A} > 0$ if $\mathbf{x}^{\top} \mathbf{A} \mathbf{x} > 0$ for all non-zero $\mathbf{x}$ in $\mathbb{R}^{n}$.
    \item[] $\mathbf{A}$ is positive semi-definite, and we note $\mathbf{A} \geq 0$ if $\mathbf{x}^{\top} \mathbf{A} \mathbf{x} \geq 0$ for all non-zero $\mathbf{x}$ in $\mathbb{R}^{n}$.
  \end{itemize}
  \item Let $N = \{ -n+1, \dots, n-1 \}$ and $M = \{ -m+1, \dots, m-1 \}$
\end{itemize}

\sectionapp{Discussion on the Convolution Operation}

\subsectionapp{Convolution as Matrix Multiplication}\label{appendix-sec:conv_matrix_multiplication}

A discrete convolution between a signal $\mathbf{x}$ and a kernel $\mathbf{k}$ can be expressed as a  product between the vectorization of $\mathbf{x}$ and a doubly-block Toeplitz matrix $\textbf{M}$, whose coefficients have been chosen to match the convolution $\mathbf{x} * \mathbf{k}$. For a 2-dimensional signal $\mathbf{x} \in \mathbb{R}^{n \times n}$ and a kernel $\mathbf{k} \in \mathbb{R}^{m \times m}$ with $m$ odd, the convolution operation can be written as follows:
\begin{equation} \label{appendix-eq:equation_conv_as_matrix}
    \reshape(\mathbf{y}) = \reshape(\pad(\mathbf{x}) * \mathbf{k}) = \mathbf{M} \reshape(\mathbf{x})
\end{equation}
where $\mathbf{M}$ is a $n^2$-by-$n^2$  doubly-block Toeplitz matrix, {\em i.e.} a block Toeplitz matrix where the blocks are also Toeplitz. (Note that this is not a doubly-block circulant matrix because of the padding.), $\mathbf{y}$ is the output of size $q \times q$ with $q = n - m + 2p + 1$, (see {\em e.g.} \citealt{dumoulin2016guide}). The $\reshape: \mathbb{R}^{n \times n} \rightarrow \mathbb{R}^{n^2}$ operator is defined as follows: $\reshape(\mathbf{x})_q = \mathbf{x}_{\lfloor q/n \rfloor,\ q\mod n}$. The $\pad: \mathbb{R}^{n \times n} \rightarrow \mathbb{R}^{(n+2p) \times (n+2p)}$ operator is a zero-padding operation which takes a signal $\mathbf{x}$ of shape $\mathbb{R}^{n \times n}$ and adds $0$ on the edges so as to obtain a new signal $\mathbf{y}$ of shape $\mathbb{R}^{(n+2p) \times (n+2p)}$. In order to have the same shape between the convoluted signal and the signal, we set $ p = \lfloor m/2 \rfloor$ \footnote{We take a square signal and an odd size square kernel to simplify the notation but the same applies for any input and kernel size. Also, we take a specific padding in order to have the same size between the input and output signal. But everything in the paper can be generalized to any paddings.}.

We now present an example of the convolution operation with doubly-block Toeplitz matrix. Let us define a kernel $\mathbf{k} \in \mathbb{R}^{3\times3}$ as follows:
\begin{equation}
  \mathbf{k} = \begin{pmatrix}
    k_{0} & k_{1} & k_{2} \\
    k_{3} & k_{4} & k_{5} \\
    k_{6} & k_{7} & k_{8} 
  \end{pmatrix}
\end{equation}
If we set the padding to 1, then, the matrix $\mathbf{M}$ is a tridiagonal doubly-block Toeplitz matrix of size $n \times n$ and has the following form:
\begin{equation}
    \mathbf{M} = \begin{pmatrix}
    \mathbf{T}_0 & \mathbf{T}_{1} &  &  & 0  \\
    \mathbf{T}_{2} & \mathbf{T}_0 & \mathbf{T}_{1} &  &  \\
     & \mathbf{T}_{2} & \scalebox{.70}{$\ddots$} & \scalebox{.70}{$\ddots$} &   \\
     &  & \scalebox{.70}{$\ddots$} & \mathbf{T}_0 & \mathbf{T}_{1}  \\
    0 &  &  & \mathbf{T}_{2} & \mathbf{T}_0  \\
    \end{pmatrix}
    \label{appendix-eq:operator_matrix}
\end{equation}
where $\mathbf{T}_j$ are banded Toeplitz matrices and the values of $\mathbf{k}$ are distributed in the Toeplitz blocks as follow:
\begin{align}
\mathbf{T}_0 = \begin{psmallmatrix}
    k_{4} & k_{3} &  &  &  0 \\
    k_{5} & k_{4} & k_{3} &  &   \\
     & k_{5} & \scalebox{.40}{$\ddots$} & \scalebox{.40}{$\ddots$}  \\
     &  &  \scalebox{.40}{$\ddots$} & k_{4} & k_{3}  \\
    0 &  &  & k_{5} & k_{4}  \\
    \end{psmallmatrix} &&
\mathbf{T}_{1} = \begin{psmallmatrix}
    k_{7} & k_{6} &  &  &  0 \\
    k_{8} & k_{7} & k_{6} &  &   \\
     & k_{8} & \scalebox{.40}{$\ddots$} & \scalebox{.40}{$\ddots$} &    \\
     &  &  \scalebox{.40}{$\ddots$} & k_{7} & k_{6}  \\
    0 &  &  & k_{8} & k_{7}  \\
    \end{psmallmatrix} &&
\mathbf{T}_{2} = \begin{psmallmatrix}
    k_{1} & k_{0} &  &  &  0 \\
    k_{2} & k_{1} & k_{0} &  &   \\
     & k_{2} & \scalebox{.40}{$\ddots$} & \scalebox{.40}{$\ddots$} &    \\
     &  &  \scalebox{.40}{$\ddots$} & k_{1} & k_{0}  \\
    0 &  &  & k_{2} & k_{1}  \\
    \end{psmallmatrix} 
\end{align}

\paragraph{Remark 1: } Note that the size of the operator matrix $\mathbf{M}$ of a convolution operation depends on the size of the signal. If a signal $\mathbf{x}$ has size $n \times n$, the vectorized signal will be of size $n^2$ and the operator matrix will be of size $n^2 \times n^2$ which can be very large. Indeed, in deep learning practice the size of the images used for training can range from 32 (CIFAR-10) to hundred for high definition images (ImageNet). Therefore, with classical methods, computing the singular values of this operator matrix can be very expensive.

\paragraph{Remark 2: } In the particular case of zero padding convolution operation, the operator matrix is a Toeplitz block with circulant block (i.e. each block of the Toeplitz block is a circulant matrix) which is a particular case of doubly-block Toeplitz matrices. 

\subsectionapp{Generating a Toeplitz matrix and block Toeplitz matrix from a trigonometric polynomial}
An $n\times n$ Toeplitz matrix $\mathbf A$ is fully determined by a two-sided sequence of scalars: $\{a_\seqidx\}_{\seqidx \in \seqsetN}$, whereas an $nm\times nm$ block Toeplitz matrix $\mathbf B$ is fully determined by a two-sided sequence of blocks $\{\mathbf B_\seqidx\}_{\seqidx \in \seqsetN}$ and where each block $\mathbf B_\seqidx$ is an $m \times m$ matrix.  

\begin{equation}
    \mathbf{A} \triangleq  \begin{psmallmatrix}
      a_0 & a_{1}   & a_{2} & \cdots & \cdots & a_{n-1}  \\
      a_{-1} & a_0 & a_{1} & \ddots & & \vdots \\
      a_{-2} & a_{-1} & \ddots & \ddots & \ddots & \vdots \\ 
     \vdots & \ddots & \ddots & \ddots & a_{1} & a_{2}\\
     \vdots & & \ddots & a_{-1} & a_{0} & a_{1} \\
    a_{-n+1} & \cdots & \cdots & a_{-2} & a_{-1} & a_0
    \end{psmallmatrix} \quad \quad
    \mathbf{B} \triangleq  \begin{psmallmatrix}
      \mathbf{B}_0 & \mathbf{B}_{1}   & \mathbf{B}_{2} & \cdots & \cdots & \mathbf{B}_{n-1}  \\
      \mathbf{B}_{-1} & \mathbf{B}_0 & \mathbf{B}_{1} & \ddots & & \vdots \\
      \mathbf{B}_{-2} & \mathbf{B}_{-1} & \ddots & \ddots & \ddots & \vdots \\ 
     \vdots & \ddots & \ddots & \ddots & \mathbf{B}_{1} & \mathbf{B}_{2}\\
     \vdots & & \ddots & \mathbf{B}_{-1} & \mathbf{B}_{0} & \mathbf{B}_{1} \\
    \mathbf{B}_{-n+1} & \cdots & \cdots & \mathbf{B}_{-2} & \mathbf{B}_{-1} & \mathbf{B}_0
    \end{psmallmatrix}
\end{equation}

The trigonometric polynomial that {\em generates} the Toeplitz matrix $\mathbf{A}$ can be defined as follows:
\begin{equation}
    f_{\mathbf{A}}(\omega) \triangleq \sum_{h \in N} a_h e^{\ci h \omega}
\end{equation}
The function $f_{\mathbf{A}}$ is said to be the \emph{generating function} of $\mathbf{A}$. To recover the Toeplitz matrix from its generating function, we have the following operator presented in the main paper:
\begin{equation} \label{appendix-eq:toeplitz_operator}
    \leftmat \mathbf{T}(f) \rightmat_{i, j} \triangleq  \frac{1}{2\pi} \int_0^{2\pi} e^{-\ci (i - j)\omega} f(\omega) \,d\omega .
\end{equation}
We can now show that $\mathbf{T}(f_{\mathbf{A}}) = \mathbf{A}$: 
\begin{align}
    \leftmat \mathbf{T}(f_\mathbf{A}) \rightmat_{i, j} &= \frac{1}{2\pi} \int_0^{2\pi} e^{-\ci (i-j)\omega} f_{\mathbf{A}}(\omega) \,d\omega  \\
    &= \frac{1}{2\pi} \int_0^{2\pi} e^{-\ci (i-j) \omega} \sum_{h \in N} a_h e^{\ci h \omega} \,d\omega  \\
    &= \frac{1}{2\pi} \int_0^{2\pi} \sum_{h \in N} a_h e^{\ci (j - i + h) \omega} \,d\omega  \\
    &= \sum_{h \in N} a_h \frac{1}{2\pi} \int_0^{2\pi} e^{\ci (j - i + h) \omega} \,d\omega 
    = a_{j-i} .
\end{align}
Because:
\begin{equation}
    \frac{1}{2\pi} \int_0^{2\pi} e^{\ci k \omega} \,d\omega = \left\{
                \begin{array}{ll}
                  1, & \text{if}\ k = 0, \\
                  0, & \text{if}\ k\ \text{is a non-zero integer number.}
                \end{array}
                \right.
\end{equation}

The same reasoning can be applied to block Toeplitz matrices. Instead of being complex-valued, the trigonometric polynomial that {\em generates} the block Toeplitz $\mathbf{B}$ is matrix-valued and can be defined as follows:
\begin{equation}
    f_{\mathbf{B}}(\omega) \triangleq \sum_{h \in N} \mathbf{B}_h e^{\ci h \omega}
\end{equation}
The function $f_{\mathbf{B}}$ is said to be the \emph{generating function} of $\mathbf{B}$. To recover the block Toeplitz matrix from its generating function, we use the Toeplitz operator defined in Equation~\ref{appendix-eq:toeplitz_operator}. We can show that $\mathbf{T}(f_\mathbf{B}) = \mathbf{B}$:
\begin{align}
    \leftmat \mathbf{T}(f_\mathbf{B}) \rightmat_{i, j} &= \frac{1}{2\pi} \int_0^{2\pi} e^{-\ci (i-j)\omega} f_{\mathbf{B}}(\omega) \,d\omega  \\
    &= \frac{1}{2\pi} \int_0^{2\pi} e^{-\ci (i-j) \omega} \sum_{h \in N} \mathbf{B}_h e^{\ci h \omega} \,d\omega  \\
    &= \frac{1}{2\pi} \int_0^{2\pi} \sum_{h \in N} \mathbf{B}_h e^{\ci (j - i + h) \omega} \,d\omega  \\
    &= \sum_{h \in N} \mathbf{B}_h \frac{1}{2\pi} \int_0^{2\pi} e^{\ci (j - i + h) \omega} \,d\omega 
    = \mathbf{B}_{j-i} .
\end{align}

\sectionapp{Main proofs}

\subsectionapp{Proof of Theorem~\ref{th:doubly_block_teoplitz_sup_singular} -- Bound on the Maximal Singular Value of Doubly-Block Toeplitz matrices}

As presented in the main paper, the Toeplitz operator can be extended to doubly-block Toeplitz matrices. 
The operator $\mathbf{D}$ maps a function $f: \mathbb{R}^2 \rightarrow \mathbb{C}$ to a doubly-block Toeplitz matrix of size $nm \times nm$. For the sake of clarity, the dependence of $\mathbf{D}(f)$  on $m$ and $n$ is omitted. Let $\mathbf{D}(f) =\leftmat\mathbf{D}_{i,j}(f)\rightmat_{i,j\in\{0, \ldots, n-1 \}}$ where $\mathbf{D}_{i,j}$ is defined as:
\begin{equation} 
    \mathbf{D}_{i,j}(f) =\leftmat\frac{1}{4\pi^{2}}\int_{[0,2\pi]^{2}}e^{-\mathbf{i}\left((i-j)\omega_{1}+(k-l)\omega_{2}\right)}f(\omega_{1},\omega_{2})\,d(\omega_{1},\omega_{2})\rightmat_{k,l\in\{0, \ldots, m-1\}} .
\end{equation}

Note that in the following, we only consider generating functions as trigonometric polynomials with real coefficients therefore the matrices generated by $\mathbf{D}(f)$ are real. We can now combine Theorems 1 and 2 to bound the maximal singular value of a doubly-block Toeplitz Matrix. 

\begin{customtheorem}{3}[Bound on the Maximal Singular Value of a Doubly-Block Toeplitz Matrix] \label{th:doubly_block_teoplitz_sup_singular}
Let $\mathbf{D}(f) \in \mathbb{R}^{nm \times nm}$ be a doubly-block Toeplitz matrix generated by the function $f$, then:
\begin{align}
\sigma_{1} \left( \mathbf{D}(f) \right) &\leq \sup_{\omega_1, \omega_2 \in [0, 2\pi]^2}|f(\omega_1,\omega_2)|
\end{align}
where the function $f: \mathbb{R}^2 \rightarrow \mathbb{C}$, is a multivariate trigonometric polynomial of the form
\begin{equation} \label{eq:muli_variate_poly_on_M}
    f(\omega_1, \omega_2) \triangleq \sum_{h_1 \in \seqsetN} \sum_{h_2 \in \seqsetM} d_{h_1, h_2} e^{\ci (h_1 \omega_1 + h_2 \omega_2)},
\end{equation}
where $d_{h_{1},h_{2}}$
is the ${h_2}^\textrm{th}$ scalar of the ${h_1}^\textrm{th}$ block of the doubly-Toeplitz matrix $\mathbf{D}(f)$.
\end{customtheorem}

\begin{proof}

A doubly-block Toeplitz matrix is by definition a block matrix where each block is a Toeplitz matrix. We can then express a doubly-block Toeplitz matrix with the operator $\mathbf{T}(F)$ where the matrix-valued generating function $F$ has Toeplitz coefficient. Let us define a matrix-valued trigonometric polynomial $F:\mathbb{R}\rightarrow\mathbb{C}^{n \times n}$ of the form:
\begin{equation}
    F(\omega_1) = \sum_{h_1 \in N} \mathbf{A}_{h_1}e^{\ci h_1\omega_1}
\end{equation}
where $\mathbf{A}_{h_1}$ are Toeplitz matrices of size $m \times m$ determined by the sequence $\{d_{h_1, -m+1}, \dots, d_{h_1, m-1} \}$. 
From Theorem 2, we have:
\begin{equation}
  \sigma_1\left(\mathbf{T}(F) \right) \leq \sup_{\omega_1 \in [0,2\pi] } \sigma_{1}\left( F(\omega_1) \right) \label{appendix-eq:th_bound_block_toeplitz}
\end{equation}

Because Toeplitz matrices are closed under addition and scalar product, $F(\omega_1)$ is also a Toeplitz matrix of size $m \times m$. 
We can thus define a function $f:\mathbb{R}^{2} \rightarrow \mathbb{C}$ such that $f(\omega_1,\ \cdot\ )$ is the generating function of $F(\omega_1)$. From Theorem 1, we can write:
\begin{align}
    \sigma_{1}\left( F(\omega_1) \right) &\leq \sup_{\omega_2 \in [0,2\pi]} \left| f(\omega_1, \omega_2) \right| \\
    \Leftrightarrow \sup_{\omega_1 \in [0,2\pi]} \sigma_{1}\left( F(\omega_1) \right) &\leq  \sup_{\omega_1, \omega_2 \in [0,2\pi]^2} \left| f(\omega_1, \omega_2) \right| \\
    \Leftrightarrow \sigma_1\left(\mathbf{T}(F) \right) &\leq \sup_{\omega_1, \omega_2 \in [0,2\pi]^2} \left| f(\omega_1, \omega_2) \right|
\end{align}
where the function $f$ is of the form:
\begin{equation}
f(\omega_1,\omega_2) = \sum_{h_1 \in N} \sum_{h_2 \in M} d_{h_1, h_2} e^{\ci \left( h_1\omega_1 + h_2 \omega_2\right)}
\end{equation}
Because the function $f(\omega_1,\ \cdot\ )$ is the generating function of $F(\omega_1)$ is it easy to show that the function $f$ is the generating function of $\mathbf{T}(F)$. Therefore, $\mathbf{T}(F) = \mathbf{D}(f)$ which concludes the proof. 
\end{proof}

\subsectionapp{Proof of Theorem~\ref{th:theorem4} -- Bound on the Maximal Singular Value of Stacked Doubly-Block Toeplitz Matrices}

In order to prove Theorem~\ref{th:theorem4}, we will need the following lemmas:

\begin{lemma}[\citealp{zhang2011matrix}] \label{appendix-th:diff_positive_semidefinite_matrices}
Let $\mathbf{A}$ and $\mathbf{B}$ be Hermitian positive semi-definite matrices. If $\mathbf{A} - \mathbf{B}$ is positive semi-definite, then:
\begin{equation*}
    \lambda_1 \left( \mathbf{B} \right) \leq \lambda_1 \left( \mathbf{A} \right)
\end{equation*}
\end{lemma}

\begin{lemma}[\citealp{serra1994preconditioning}] \label{appendix-th:block_toeplitz_positive_definite}
If the doubly-block Toeplitz matrix $\mathbf{D}(f)$ is generated by a non-negative function $f$ not identically zero, then the matrix $\mathbf{D}(f)$ is positive definite. 
\end{lemma}

\begin{lemma}[\citealp{serra1994preconditioning}] \label{appendix-th:block_toeplitz_hermitian}
If the doubly-block Toeplitz matrix $\mathbf{D}(f)$ is generated by a function $f: \mathbb{R}^2 \rightarrow \mathbb{R}$, then the matrix $\mathbf{D}(f)$ is Hermitian. 
\end{lemma}

\begin{lemma}[\citealp{gutierrez2012block}] \label{appendix-th:properties_block_toeplitz}
Let $f:\mathbb{R}^2 \rightarrow \mathbb{C}$ and $g:\mathbb{R}^2 \rightarrow \mathbb{C}$ be two continuous and $2\pi$-periodic functions. Let $\mathbf{D}(f)$ and $\mathbf{D}(g)$ be doubly-block Toeplitz matrices generated by the function $f$ and $g$ respectively. Then:
\begin{itemize}
    \item $\mathbf{D}^\top(f) = \mathbf{D}(f^*)$
    \item $\mathbf{D}(f) + \mathbf{D}(g) = \mathbf{D}(f + g)$
\end{itemize}
\end{lemma}

Before proving Theorem~\ref{th:theorem4}, we generalize the famous Widom identity \cite{widom1976asymptotic} that express the relation between Toeplitz and Hankel matrix to doubly-block Toeplitz and Hankel matrices. 
We will need to generalize the doubly-block Toeplitz operator presented in the paper. From now on, without loss of generality, we will assume that $n=m$ to simplify notations. 
Let $\mathbf{H}^{\alpha_p} (f) = \leftmat \mathbf{H}^{\alpha_p}_{i,j}(f) \rightmat_{i,j \in \{0 \ldots n-1 \}}$ where $\mathbf{H}^{\alpha_p}_{i,j}$ is defined as:

\begin{equation}
\mathbf{H}^{\alpha_p}_{i,j}(f) =\leftmat \frac{1}{4\pi^{2}}\int_{[0,2\pi]^{2}} e^{-\mathbf{i} \alpha_p(i, j, k, l, \omega_1, \omega_2)}  f(\omega_{1},\omega_{2})\,d(\omega_{1},\omega_{2})
\rightmat_{k,l \in \{0, \ldots, n-1 \}} .
\end{equation}
Note that as with the operator $\mathbf{D}(f)$ we only consider generating functions as trigonometric polynomials with real coefficients therefore the matrices generated by $\mathbf{H}(f)$ are real. 

We will use the following $\alpha$ functions:
\begin{itemize}
  \item[] $\alpha_0(i, j, k, l, \omega_1, \omega_2) = (-j-i-1)\omega_1 + (k-l)\omega_2$
  \item[] $\alpha_1(i, j, k, l, \omega_1, \omega_2) = (i-j)\omega_1 + (-l-k-1)\omega_2$
  \item[] $\alpha_2(i, j, k, l, \omega_1, \omega_2) = (-j-i-1)\omega_1 + (-l-k-1)\omega_2$
  \item[] $\alpha_3(i, j, k, l, \omega_1, \omega_2) = (-j-i+n)\omega_1 + (-l-k-1)\omega_2$
\end{itemize}
As with the doubly-block Toeplitz operator $\mathbf{D}(f)$, the matrices generated by the operator $\mathbf{H}^{\alpha_p}$ are of size $n^2 \times n^2$. 

We now present the generalization of the Widom identity for Doubly-Block Toeplitz matrices below:
\begin{lemma} \label{appendix-th:widom_idenity} Let $f:\mathbb{R}^2 \rightarrow \mathbb{C}$ and $g:\mathbb{R}^2 \rightarrow \mathbb{C}$ be two continuous and $2\pi$-periodic functions. 
We can decompose the Doubly-Block Toeplitz matrix $\mathbf{D}(fg)$ as follows:
\begin{equation}
    \mathbf{D}(fg) = \mathbf{D}(f)\mathbf{D}(g) + \sum_{p=0}^3 \mathbf{H}^{\alpha_p \top}(f^*) \mathbf{H}^{\alpha_p}(g) + \mathbf{Q} \left( \sum_{p=0}^3 \mathbf{H}^{\alpha_p \top}(f) \mathbf{H}^{\alpha_p }(g^*) \right) \mathbf{Q}.
\end{equation}
where $\mathbf{Q}$ is the anti-identity matrix of size $n^2 \times n^2$.
\end{lemma}
\begin{proof}

Let $(i, j)$ be matrix indexes such $(\ \cdot\ )_{i, j}$ correspond to the value at the $i^\textrm{th}$ row and $j^\textrm{th}$ column, let us define the following notation:
\begin{align*}
    i_1 &= \left\lfloor i/n \right\rfloor \quad \quad &&j_1 = \left\lfloor j/n \right\rfloor \\
    i_2 &= i \mod n \quad \quad &&j_2 = j \mod n
\end{align*}

Let us define $\hat{f}$ as the 2 dimensional Fourier transform of the function $f$. We refer to $\hat{f}_{h_1, h_2}$ as the Fourier coefficient indexed by $(h_1, h_2)$ where $h_1$ correspond to the index of the block of the doubly-block Toeplitz and $h_2$ correspond to the index of the value inside the block. More precisely, we have 
\begin{align}
    \leftmat \mathbf{D}(f) \rightmat_{i, j} &= \hat{f}_{(\left\lfloor j/n \right\rfloor - \left\lfloor i/n \right\rfloor), ((j \mod n) - (i \mod n)))} \label{appendix-eq:expression_fourier} \\
    \leftmat \mathbf{H}^{\alpha_0}(f) \rightmat_{i, j} &= \hat{f}_{(\left\lfloor j/n \right\rfloor + \left\lfloor i/n \right\rfloor + 1), ((j \mod n) - (i \mod n)))} \\
    \leftmat \mathbf{H}^{\alpha_1}(f) \rightmat_{i, j} &= \hat{f}_{(\left\lfloor j/n \right\rfloor - \left\lfloor i/n \right\rfloor), ((j \mod n) + (i \mod n) + 1))} \\
    \leftmat \mathbf{H}^{\alpha_2}(f) \rightmat_{i, j} &= \hat{f}_{(\left\lfloor j/n \right\rfloor - \left\lfloor i/n \right\rfloor), ((j \mod n) - (i \mod n)))} \\
    \leftmat \mathbf{H}^{\alpha_3}(f) \rightmat_{i, j} &= \hat{f}_{(\left\lfloor j/n \right\rfloor + \left\lfloor i/n \right\rfloor + n), ((j \mod n) + (i \mod n) + 1))}
\end{align}

We simplify the notation of the expressions above as follow:
\begin{align}
    \leftmat \mathbf{D}(f) \rightmat_{i, j} &= \hat{f}_{(j_1 - i_1), (j_2 - i_2 )} \\
    \leftmat \mathbf{H}^{\alpha_0}(f) \rightmat_{i, j} &= \hat{f}_{(j_1 + i_1 + 1), (j_2 - i_2 )} \\
    \leftmat \mathbf{H}^{\alpha_1}(f) \rightmat_{i, j} &= \hat{f}_{(j_1 - i_1), (j_2 + i_2 + 1)} \\
    \leftmat \mathbf{H}^{\alpha_2}(f) \rightmat_{i, j} &= \hat{f}_{(j_1 - i_1), (j_2 - i_2 )} \\
    \leftmat \mathbf{H}^{\alpha_3}(f) \rightmat_{i, j} &= \hat{f}_{(j_1 + i_1 + n), (j_2 + i_2 + 1)}
\end{align}

The convolution theorem states that the Fourier transform of a product of two functions is the convolution of their Fourier coefficients. Therefore, one can observe that the entry $(i, j)$ of the matrix $\mathbf{D}(f g)$ can be express as follows:
\begin{equation*}
    \leftmat \mathbf{D}(f g) \rightmat_{i, j} = \sum_{k_1 = -2n + 1}^{2n-1} \sum_{k_2 = -2n + 1}^{2n-1} \hat{f}_{(k_1-i_1),(k_2-i_2)} \hat{g}_{(j_1-k_1),(j_2-k_2)}. 
\end{equation*}

By splitting the double sums and simplifying, we obtain:
\begin{align} \label{eq:split_double_sum}
\left( \mathbf{D}(f g) \right)_{i, j} &= 
\sum_{k_1, k_2 \in P} \left(
\hat{f}_{(k_1-i_1),(k_2-i_2)} \hat{g}_{(j_1-k_1),(j_2-k_2)} +
\hat{f}_{(-k_1-i_1-1),(k_2-i_2)} \hat{g}_{(j_1+k_1+1),(j_2-k_2)} \right. \notag \\ &\quad+ \left.
\hat{f}_{(k_1-i_1),(-k_2-i_2-1)} \hat{g}_{(j_1-k_1),(j_2+k_2+1)} +
\hat{f}_{(-k_1-i_1-1),(-k_2-i_2-1)} \hat{g}_{(j_1+k_1+1),(j_2+k_2+1)} \right. \notag \\ &\quad+ \left.
\hat{f}_{(k_1-i_1+n),(-k_2-i_2-1)} \hat{g}_{(j_1-k_1-n),(j_2+k_2+1)} +
\hat{f}_{(k_1-i_1+n),(k_2-i_2)} \hat{g}_{(j_1-k_1-n),(j_2-k_2)} \right. \notag \\ &\quad+ \left.
\hat{f}_{(k_1-i_1),(k_2-i_2+n)} \hat{g}_{(j_1-k_1),(j_2-k_2-n)} +
\hat{f}_{(k_1-i_1+n),(k_2-i_2+n)} \hat{g}_{(j_1-k_1-n),(j_2-k_2-n)} \right. \notag \\ &\quad+ \left.
\hat{f}_{(-k_1-i_1-1),(k_2-i_2+n)} \hat{g}_{(j_1+k_1+1),(j_2-k_2-n)}  \right)
\end{align}
where $P = \{ (k_1, k_2)\ |\ k_1, k_2 \in \mathbb{N} \cup 0, 0 \leq k_1 \leq n-1,  0 \leq k_2 \leq n-1 \}$.

Furthermore, we can observe the following:
\begin{equation*}
    \leftmat \mathbf{D}(f) \mathbf{D}(g) \rightmat_{i, j} = \sum_{k = 0}^{n^2} \leftmat\mathbf{D}(f)\rightmat_{i, k} \leftmat\mathbf{D}(g)\rightmat_{k, j}  = \sum_{k_1, k_2 \in P} \hat{f}_{(k_1-i_1),(k_2-i_2)} \hat{g}_{(j_1-k_1),(j_2-k_2)}
\end{equation*}

{\allowdisplaybreaks
\begin{flalign*}
    \leftmat \mathbf{H}^{\alpha_1 \top}(f^*) \mathbf{H}^{\alpha_1}(g) \rightmat_{i, j} &=  \sum_{k_1, k_2 \in P} \hat{f}^*_{(k_1+i_1+1),(i_2-k_2)} \hat{g}_{(j_1+k_1+1),(j_2-k_2)} \\
    &=  \sum_{k_1, k_2 \in P} \hat{f}_{(-k_1-i_1-1),(k_2-i_2)} \hat{g}_{(j_1+k_1+1),(j_2-k_2)} \\[10pt]
    \leftmat \mathbf{H}^{\alpha_2 \top}(f^*) \mathbf{H}^{\alpha_2}(g) \rightmat_{i, j} &=  \sum_{k_1, k_2 \in P} \hat{f}^*_{(i_1-k_1),(k_2+i_2+1)} \hat{g}_{(j_1-k_1),(j_2+k_2+1)} \\
    &=  \sum_{k_1, k_2 \in P} \hat{f}_{(k_1-i_1),(-k_2-i_2-1)} \hat{g}_{(j_1-k_1),(j_2+k_2+1)} \\[10pt]
    \leftmat \mathbf{H}^{\alpha_3 \top}(f^*) \mathbf{H}^{\alpha_3}(g) \rightmat_{i, j} &=  \sum_{k_1, k_2 \in P} \hat{f}^*_{(k_1+i_1+1),(k_2+i_2+1)} \hat{g}_{(j_1+k_1+1),(k_2+j_2+1)} \\
    &= \sum_{k_1, k_2 \in P} \hat{f}_{(-k_1-i_1-1),(-k_2-i_2-1)} \hat{g}_{(j_1+k_1+1),(k_2+j_2+1)} \\[10pt]
    \leftmat \mathbf{H}^{\alpha_4 \top}(f^*) \mathbf{H}^{\alpha_4}(g) \rightmat_{i, j} &= \sum_{k_1, k_2 \in P} \hat{f}^*_{(i_1-k_1-n),(k_2+i_2+1)} \hat{g}_{(j_1-k_1-n),(j_2+k_2+1)} \\
    &=  \sum_{k_1, k_2 \in P} \hat{f}_{(k_1-i_1+n),(-k_2-i_2-1)} \hat{g}_{(j_1-k_1-n),(j_2+k_2+1)} \\
\end{flalign*}
}
Let us define the matrix $\mathbf{Q}$ of size $n^2 \times n^2$ as the anti-identity matrix. We have the following:

{\allowdisplaybreaks
\begin{flalign*}
    \leftmat \mathbf{H}^{\alpha_1 \top}(f) \mathbf{H}^{\alpha_1}(g^*) \rightmat_{i, j} &= \sum_{k_1, k_2 \in P} \hat{f}_{(k_1+i_1+1),(i_2-k_2)} \hat{g}^*_{(j_1+k_1+1),(j_2-k_2)} \\
    &= \sum_{k_1, k_2 \in P} \hat{f}_{(k_1+i_1+1),(i_2-k_2)} \hat{g}_{(-j_1-k_1-1),(k_2-j_2)} \\
    \Leftrightarrow \leftmat \mathbf{Q} \mathbf{H}^{\alpha_1 \top}(f) \mathbf{H}^{\alpha_1}(g^*) \mathbf{Q} \rightmat_{i, j} &= \sum_{k_1, k_2 \in P} \hat{f}_{(k_1-i_1+n),(k_2-i_2)} \hat{g}_{(j_1-k_1-n),(j_2-k_2)} \\[10pt]
    \leftmat \mathbf{H}^{\alpha_2 \top}(f) \mathbf{H}^{\alpha_2}(g^*) \rightmat_{i, j} &=  \sum_{k_1, k_2 \in P} \hat{f}_{(i_1-k_1),(k_2+i_2+1)} \hat{g}^*_{(j_1-k_1),(j_2+k_2+1)} \\
    &=  \sum_{k_1, k_2 \in P} \hat{f}_{(i_1-k_1),(k_2+i_2+1)} \hat{g}_{(k_1-j_1),(-j_2-k_2-1)} \\
    \Leftrightarrow \leftmat \mathbf{Q} \mathbf{H}^{\alpha_2 \top}(f) \mathbf{H}^{\alpha_2}(g^*) \mathbf{Q} \rightmat_{i, j} &=  \sum_{k_1, k_2 \in P} \hat{f}_{(k_1-i_1),(k_2-i_2+n)} \hat{g}_{(j_1-k_1),(j_2-k_2-n)} \\[10pt]
    \leftmat \mathbf{H}^{\alpha_3 \top}(f) \mathbf{H}^{\alpha_3}(g^*) \rightmat_{i, j} &=  \sum_{k_1, k_2 \in P}  \hat{f}_{(k_1+i_1+1),(k_2+i_2+1)} \hat{g}^*_{(j_1+k_1+1),(k_2+j_2+1)} \\
    &=  \sum_{k_1, k_2 \in P} \hat{f}_{(k_1+i_1+1),(k_2+i_2+1)} \hat{g}_{(-j_1-k_1-1),(-k_2-j_2-1)} \\
    \Leftrightarrow \leftmat \mathbf{Q} \mathbf{H}^{\alpha_3 \top}(f) \mathbf{H}^{\alpha_3}(g^*) \mathbf{Q} \rightmat_{i, j} &=  \sum_{k_1, k_2 \in P} \hat{f}_{(k_1-i_1+n),(k_2-i_2+n)} \hat{g}_{(j_1-k_1-n),(-k_2+j_2-n)} \\[10pt]
    \leftmat \mathbf{H}^{\alpha_4 \top}(f) \mathbf{H}^{\alpha_4}(g^*) \rightmat_{i, j} &=  \sum_{k_1, k_2 \in P}  \hat{f}_{(-k_1+i_1-n),(k_2+i_2+1)} \hat{g}^*_{(j_1-k_1-n),(j_2+k_2+1)} \\
    &= \sum_{k_1, k_2 \in P} \hat{f}_{(-k_1+i_1-n),(k_2+i_2+1)} \hat{g}_{(-j_1+k_1+n),(-j_2-k_2-1)} \\
    \Leftrightarrow \leftmat \mathbf{Q} \mathbf{H}^{\alpha_4 \top}(f) \mathbf{H}^{\alpha_4}(g^*) \mathbf{Q} \rightmat_{i, j} &= \sum_{k_1, k_2 \in P} \hat{f}_{(-k_1-i_1-1),(k_2-i_2+n)} \hat{g}_{(j_1+k_1+1),(j_2-k_2-n)}
\end{flalign*}
}

Now, we can observe from Equation~\ref{eq:split_double_sum} that:
\begin{equation}
    \mathbf{D}(fg) = \mathbf{D}(f)\mathbf{D}(g) + \sum_{p=0}^3 \mathbf{H}^{\alpha_p \top}(f^*) \mathbf{H}^{\alpha_p}(g) + \mathbf{Q} \left( \sum_{p=0}^3 \mathbf{H}^{\alpha_p \top}(f) \mathbf{H}^{\alpha_p}(g^*) \right) \mathbf{Q}.
\end{equation}
which concludes the proof. 
\end{proof}

Now we can state our theorem which bounds the maximal singular value of vertically stacked doubly-block Toeplitz matrices with their generating functions. 
\begin{customtheorem}{4}[Bound on the maximal singular value of stacked Doubly-block Toeplitz matrices] \label{th:theorem4}
Consider doubly-block Toeplitz matrices $\mathbf{D}(f_1), \dots, \mathbf{D}(f_{\cin})$ where $f_i: \mathbb{R}^2 \rightarrow \mathbb{C}$ is a generating function. Construct a matrix $\mathbf{M}$ with $\cin\times n^2$ rows and $n^2$ columns, as follows:
\begin{equation}
    \mathbf{M} \triangleq \leftmat \mathbf{D}^\top(f_1), \dots, \mathbf{D}^\top(f_{\cin}) \rightmat^\top .
\end{equation}
Then, we can bound the maximal singular value of the matrix $\mathbf{M}$ as follows:
\begin{align}
    \sigma_1\left(\mathbf{M} \right) &\leq \sup_{\omega_1, \omega_2 \in \left[0, 2\pi\right]^2} \sqrt{ \sum_{i=1}^{\cin} \left|f_i\right (\omega_1, \omega_2)|^2} .
\end{align}
\end{customtheorem}

\begin{proof}
First, let us observe the following:
\begin{equation}
    \sigma_1^2 \left( \mathbf{M} \right) = \lambda_1 \left( \mathbf{M}^{\top} \mathbf{M} \right) = \lambda_1 \left( \sum_{i=1}^{\cin} \mathbf{D}^{\top} \left(f_i \right) \mathbf{D} (f_i) \right).
\end{equation}

And the fact that:
\begin{align}
    \lambda_1 \left( \sum_{i=1}^{\cin} \mathbf{D} \left(|f_i|^2 \right) \right) \quad &\stackrel{\text{by Lemma~\ref{appendix-th:properties_block_toeplitz}}}{=} \quad \lambda_1 \left( \mathbf{D} \left( \sum_{i=1}^{\cin} |f_i|^2 \right) \right) \\ 
    \quad &\stackrel{\text{by Lemma~\ref{appendix-th:block_toeplitz_hermitian}}}{=} \quad \sigma_1 \left( \mathbf{D} \left( \sum_{i=1}^{\cin} |f_i|^2 \right) \right) \\
    \quad &\stackrel{\text{by Theorem~\ref{th:doubly_block_teoplitz_sup_singular}}}{\leq} \quad \sup_{\omega_1, \omega_2 \in [0, 2\pi]^2} \sum_{i=1}^{\cin} |f_i(\omega_1, \omega_2)|^2.
\end{align}

To prove the Theorem, we simply need to verify the following inequality:
\begin{equation}
    \lambda_1 \left( \sum_{i=1}^{\cin} \mathbf{D}^{\top} \left(f_i \right) \mathbf{D} (f_i) \right) \leq \lambda_1 \left( \mathbf{D} \left( \sum_{i=1}^{\cin} |f_i|^2 \right) \right). \label{appendix-eq:eq2}
\end{equation}

From the positive definiteness of the following matrix:
\begin{equation}
    \mathbf{D} \left( \sum_{i=1}^{\cin} |f_i|^2 \right) - \sum_{i=1}^{\cin} \mathbf{D}^{\top} \left(f_i \right) \mathbf{D} (f_i),
\end{equation}

one can observe that the r.h.s is a real symmetric positive definite matrix by Lemma~\ref{appendix-th:block_toeplitz_positive_definite} and \ref{appendix-th:block_toeplitz_hermitian}. Furthermore, the l.h.s is a sum of positive semi-definite matrices. Therefore, if the subtraction of the two is positive semi-definite, one could apply Lemma~\ref{appendix-th:diff_positive_semidefinite_matrices} to prove the inequality
~\ref{appendix-eq:eq2}. 

We know from Lemma~\ref{appendix-th:widom_idenity} that 
\begin{equation}
    \mathbf{D}(fg) - \mathbf{D}(f)\mathbf{D}(g) = \sum_{p=0}^3 \mathbf{H}^{\alpha_p \top}(f^*) \mathbf{H}^{\alpha_p}(g) + \mathbf{Q} \left( \sum_{p=0}^3 \mathbf{H}^{\alpha_p \top}(f) \mathbf{H}^{\alpha_p}(g^*) \right) \mathbf{Q}.
\end{equation}
By instantiating $f = f^*$, $g = f$ and with the use of Lemma~\ref{appendix-th:properties_block_toeplitz}, we obtain:
\begin{align}
    \mathbf{D}(f^* f) - \mathbf{D}(f^*)\mathbf{D}(f)
    &= \mathbf{D}(|f|^2) - \mathbf{D}^{\top}(f)\mathbf{D}(f) \\
    &= \sum_{p=0}^3 \mathbf{H}^{\alpha_p \top}(f)\mathbf{H}^{\alpha_p}(f) + \mathbf{Q} \left( \sum_{p=0}^3 \mathbf{H}^{\alpha_p \top}(f^*)\mathbf{H}^{\alpha_p} (f^*) \right) \mathbf{Q} . \label{appendix-eq:widom_identity_block_topelitz}
\end{align}

From Equation~\ref{appendix-eq:widom_identity_block_topelitz}, we can see that the matrix $\mathbf{D}(|f|^2) - \mathbf{D}^{\top}(f)\mathbf{D}(f)$
is positive semi-definite because it can be decomposed into a sum of positive semi-definite matrices.
Therefore, because positive semi-definiteness is closed under addition, we have:
\begin{align}
    \sum_{i=1}^{\cin} \leftmat \mathbf{D} \left( |f_i|^2 \right) - \mathbf{D}^{\top} \left(f_i \right) \mathbf{D} (f_i) \rightmat &\geq 0
\end{align}

By re-arranging and with the use Lemma~\ref{appendix-th:properties_block_toeplitz}, we obtain:
\begin{align}
   \sum_{i=1}^{\cin} \leftmat \mathbf{D} \left( |f_i|^2 \right) \rightmat - \sum_{i=1}^{\cin} \leftmat \mathbf{D}^{\top} \left(f_i \right) \mathbf{D} (f_i) \rightmat &\geq 0 \\
    \mathbf{D} \left( \sum_{i=1}^{\cin} |f_i|^2 \right) - \sum_{i=1}^{\cin} \leftmat \mathbf{D}^{\top} \left(f_i \right) \mathbf{D} (f_i) \rightmat &\geq 0
\end{align}

We can conclude that the inequality~\ref{appendix-eq:eq2} is true and therefore by Lemma~\ref{appendix-th:diff_positive_semidefinite_matrices} we have:
\begin{align}
    \lambda_1 \left( \sum_{i=1}^{\cin} \mathbf{D}^{\top} \left(f_i \right) \mathbf{D} (f_i) \right) &\leq \lambda_1 \left( \mathbf{D} \left( \sum_{i=1}^{\cin} |f_i|^2 \right) \right) \\
    \Leftrightarrow \sigma_1^2 \left( \mathbf{M} \right) &\leq \sup_{\omega_1, \omega_2 \in [0, 2\pi]^2} \sum_{i=1}^{\cin} |f_i(\omega_1, \omega_2)|^2 \\
    \Leftrightarrow \sigma_1 \left( \mathbf{M} \right) &\leq \sup_{\omega_1, \omega_2 \in [0, 2\pi]^2} \sqrt{ \sum_{i=1}^{\cin} |f_i(\omega_1, \omega_2)|^2 }
\end{align}
which concludes the proof. 
\end{proof}

\subsectionapp{Proof of Theorem~\ref{th:theorem5} -- Bound on the Maximal Singular Value on the Convolution Operation}

First, in order to prove Theorem~\ref{th:theorem5}, we will need the following lemma which bound the singular values of a matrix constructed from the concatenation of multiple matrix. 

\begin{lemma}\label{appendix-th:bound_concatenation_matrices}
Let us define matrices $\mathbf{A}_1, \dots, \mathbf{A}_p$ with $\mathbf{A}_i \in \mathbb{R}^{n \times n}$. Let us construct the matrix $\mathbf{M} \in \mathbb{R}^{n \times pn}$ as follows:
\begin{equation}
    \mathbf{M} \triangleq \leftmat \mathbf{A}_1, \dots, \mathbf{A}_p \rightmat
\end{equation}
where $\leftmat\ \cdot\ \rightmat$ define the concatenation operation. Then, we can bound the singular values of the matrix $\mathbf{M}$ as follows:
\begin{equation}
    \sigma_1(\mathbf{M}) \leq \sqrt{\sum_{i=1}^p \sigma_1(\mathbf{A}_i)^2}
\end{equation}
\end{lemma}
\begin{proof}
\begin{align}
    \sigma_1\left(\mathbf{M}\right)^2 &= \lambda_1\left(\mathbf{M} \mathbf{M}^\top\right) \\
    &= \lambda_1\left( \sum_{i=1}^p\mathbf{A}_i \mathbf{A}_i^\top  \right) \\
    &\leq \sum_{i=1}^p \lambda_1\left( \mathbf{A}_i \mathbf{A}_i^\top  \right) \\
    &\leq \sum_{i=1}^p \sigma_1\left( \mathbf{A}_i \right)^2 \\
    \Leftrightarrow \sigma_1\left(\mathbf{M}\right) &\leq \sqrt{\sum_{i=1}^p \sigma_1(\mathbf{A}_i)^2}
\end{align}
which concludes the proof.
\end{proof}

\begin{customtheorem}{5}[\textbf{Main Result}: Bound on the maximal singular value on the convolution operation] \label{th:theorem5}
Let us define doubly-block Toeplitz matrices $\mathbf{D}(f_{11}), \dots, \mathbf{D}(f_{\cin\times \cout})$ where $f_{ij}: \mathbb{R}^2 \rightarrow \mathbb{C}$ is a generating function. Construct a matrix $\mathbf{M}$ with $\cin\times n^2$ rows and $\cout\times n^2$ columns such as
{\small
\begin{equation}
    \mathbf{M} \triangleq  \leftmat\begin{array}{ccc}
    \mathbf{D}(f_{11}) & \cdots & \mathbf{D}(f_{1,\cout})   \\
    \vdots & & \vdots   \\
    \mathbf{D}(f_{\cin,1}) & \cdots & \mathbf{D}(f_{\cin,\cout}) \\
    \end{array}\rightmat .
\end{equation}}
Then, with $f_{ij}$ a multivariate polynomial of the same form as Equation~\ref{eq:muli_variate_poly_on_M}, we have:
\begin{equation}
   \sigma_1(\mathbf{M}) \leq \sqrt{ \sum_{i=1}^{\cout} \sup_{\omega_1, \omega_2 \in [0, 2\pi]^2} \sum_{j = 1}^{\cin} \left|f_{ij}(\omega_1, \omega_2) \right|^2 } .
\end{equation}
\end{customtheorem}

\begin{proof}
Let us define the matrix $\mathbf{M}_i$ as follows:
\begin{equation}
    \mathbf{M}_i = \leftmat \mathbf{D}(f_{1, i})^\top, \dots, \mathbf{D}(f_{\cin, i})^\top \rightmat^\top .
\end{equation}
We can express the matrix $\mathbf{M}$ as the concatenation of multiple $\mathbf{M}_i$ matrices:
\begin{equation}
    \mathbf{M} = \leftmat \mathbf{M}_1, \dots, \mathbf{M}_{\cout} \rightmat
\end{equation}
Then, we can bound the singular values of the matrix $\mathbf{M}$ as follows:
\begin{align}
    \sigma_1\left(\mathbf{M}\right) &\stackrel{\text{by Lemma~\ref{appendix-th:bound_concatenation_matrices}}}{\leq} \sqrt{\sum_{i=1}^{\cout} \sigma_1(\mathbf{M}_i)^2} \\
    \sigma_1\left(\mathbf{M}\right) &\stackrel{\text{by Theorem~\ref{th:theorem4}}}{\leq} \sqrt{\sum_{j=1}^{\cout} \sup_{\omega_1, \omega_2 \in [0, 2\pi]^2} \sum_{i=1}^{\cin} |f_{ij}(\omega_1, \omega_2)|^2 }
\end{align}
which concludes the proof. 
\end{proof}

\newpage
\sectionapp{Additional Results and Discussions on the Experiments}

\begin{table}[h]
  \centering
  \caption{This table shows the efficiency of LipBound computation vs the Power Method with 10 iterations on the full networks.}
  {\footnotesize
    \begin{tabular}{llllr}
    \toprule
      &   & \multicolumn{1}{c}{\textbf{LipBound (ms)}} & \multicolumn{1}{c}{\textbf{Power Method (ms)}} & \textbf{Ratio} \\
    \midrule
     \citealt{krizhevsky2012imagenet} & \textbf{AlexNet} & \phantom{....}$4.75\pm1.1$ & \phantom{....}$38.75\pm2.52$ & \textbf{8.14 }\\
    \midrule
    \multirow{5}[2]{*}{\citealt{he2016deep}} & \textbf{ResNet 18} & \phantom{..}$29.88\pm1.73$ & \phantom{..}$148.35\pm14.92$ & \textbf{4.96 }\\
      & \textbf{ResNet 34} & \phantom{..}$54.73\pm3.62$ & \phantom{..}$266.85\pm25.35$ & \textbf{4.87 }\\
      & \textbf{ResNet 50} & \phantom{..}$60.77\pm4.62$ & \phantom{..}$467.61\pm36.52$ & \textbf{7.69 }\\
      & \textbf{ResNet 101} & $102.72\pm11.53$ & \phantom{..}$817.06\pm102.87$ & \textbf{7.95 }\\
      & \textbf{ResNet 152} & $158.80\pm20.84$ & $1373.57\pm164.37$ & \textbf{8.64} \\
    \midrule
    \multirow{4}[2]{*}{\citealt{huang2017densely}} & \textbf{DenseNet 121} & $125.55\pm14.59$ & \phantom{..}$937.35\pm11.52$ & \textbf{7.46 }\\
      & \textbf{DenseNet 161} & $176.11\pm19.13$ & $1292.61\pm30.5$ & \textbf{7.33} \\
      & \textbf{DenseNet 169} & $188.03\pm19.74$ & $1372.62\pm21.16$ & \textbf{7.29} \\
      & \textbf{DenseNet 201} & $281.13\pm23.41$ & $1930.19\pm170.79$ & \textbf{6.86 }\\
    \midrule
    \multirow{4}[2]{*}{\citealt{simonyan2014very}} & \textbf{VGG 11} & \phantom{..}$13.73\pm1.19$ & \phantom{....}$81.78\pm4.45$ & \textbf{5.95 }\\
      & \textbf{VGG 13} & \phantom{..}$14.96\pm1.99$ & \phantom{..}$102.04\pm4.2$ & \textbf{6.82 }\\
      & \textbf{VGG 16} & \phantom{..}$21.92\pm1.94$ & \phantom{..}$132.29\pm5.99$ & \textbf{6.03 }\\
      & \textbf{VGG 19} & \phantom{..}$29.05\pm0.66$ & \phantom{..}$162.28\pm4.87$ & \textbf{5.58 }\\
    \midrule
    \citealt{zagoruyko2016wide} & \textbf{WideResnet 50-2} & $113.28\pm45.44$ & \phantom{..}$468.74\pm6.54$ & \textbf{4.13 }\\
    \midrule
    \multirow{2}[2]{*}{\citealt{iandola2016squeezenet}} & \textbf{SqueezeNet 1-0} & \phantom{..}$18.44\pm5.93$ & \phantom{....}$222.4\pm25.49$ & \textbf{12.05} \\
      & \textbf{SqueezeNet 1-1} & \phantom{..}$18.26\pm6.65$ & \phantom{....}$209.8\pm3.59$ & \textbf{11.48} \\
    \bottomrule
    \end{tabular}%
    }
  \label{tab:efficiency_lipbound_full_model}%
\end{table}%

The comparison of Table 1 of the main paper has been made with the following code provided by the authors: 
\begin{itemize}
  \item \makebox[4.6cm]{\citet{sedghi2018the}\hfill} \url{https://github.com/brain-research/conv-sv}
  \item \makebox[4.6cm]{\citet{singla2019bounding}\hfill} \url{https://github.com/singlasahil14/CONV-SV}
  \item \makebox[4.6cm]{\citet{farnia2018generalizable}\hfill} \url{https://github.com/jessemzhang/dl_spectral_normalization}
\end{itemize}
We translated the code of \citet{sedghi2018the} from TensorFlow to PyTorch in order to use the PyTorch CUDA Profiler.
We extended the experiments presented in Table 1 with Table~\ref{tab:efficiency_lipbound_full_model}.
This table shows the efficiency of LipBound computation vs the Power Method with 10 iterations on the full network (\ie on all the convolutions of each network). The Ratio represents the \emph{speed gain} between our proposed method and the Power Method.

%% file: main.bbl
\begin{thebibliography}{39}
\providecommand{\natexlab}[1]{#1}
\providecommand{\url}[1]{\texttt{#1}}
\providecommand{\urlprefix}{URL }
\expandafter\ifx\csname urlstyle\endcsname\relax
  \providecommand{\doi}[1]{doi:\discretionary{}{}{}#1}\else
  \providecommand{\doi}{doi:\discretionary{}{}{}\begingroup
  \urlstyle{rm}\Url}\fi

\bibitem[{Arjovsky, Chintala, and Bottou(2017)}]{arjovsky2017wasserstein}
Arjovsky, M.; Chintala, S.; and Bottou, L. 2017.
\newblock Wasserstein gan.
\newblock \emph{arXiv preprint arXiv:1701.07875} .

\bibitem[{Bartlett, Foster, and Telgarsky(2017)}]{bartlett2017spectrally}
Bartlett, P.~L.; Foster, D.~J.; and Telgarsky, M.~J. 2017.
\newblock Spectrally-normalized margin bounds for neural networks.
\newblock In \emph{Advances in Neural Information Processing Systems
  (NeurIPS)}.

\bibitem[{Carlini et~al.(2019)Carlini, Athalye, Papernot, Brendel, Rauber,
  Tsipras, Goodfellow, and Madry}]{carlini2019evaluating}
Carlini, N.; Athalye, A.; Papernot, N.; Brendel, W.; Rauber, J.; Tsipras, D.;
  Goodfellow, I.; and Madry, A. 2019.
\newblock On Evaluating Adversarial Robustness.
\newblock \emph{arXiv preprint arXiv:1902.06705} .

\bibitem[{Carlini and Wagner(2017)}]{carlini2017towards}
Carlini, N.; and Wagner, D. 2017.
\newblock Towards evaluating the robustness of neural networks.
\newblock In \emph{2017 ieee symposium on security and privacy (sp)}, 39--57.
  IEEE.

\bibitem[{Cisse et~al.(2017)Cisse, Bojanowski, Grave, Dauphin, and
  Usunier}]{cisse2017parseval}
Cisse, M.; Bojanowski, P.; Grave, E.; Dauphin, Y.; and Usunier, N. 2017.
\newblock Parseval Networks: Improving Robustness to Adversarial Examples.
\newblock In \emph{Proceedings of the 34th International Conference on Machine
  Learning (ICML)}.

\bibitem[{De~La~Chevrotiere(2009)}]{de2009finding}
De~La~Chevrotiere, G. 2009.
\newblock Finding the maximum modulus of a polynomial on the polydisk using a
  generalization of steckins lemma.
\newblock \emph{SIAM Undergraduate Research Online} .

\bibitem[{Dumoulin and Visin(2016)}]{dumoulin2016guide}
Dumoulin, V.; and Visin, F. 2016.
\newblock A guide to convolution arithmetic for deep learning.
\newblock \emph{arXiv preprint arXiv:1603.07285} .

\bibitem[{Farnia, Zhang, and Tse(2019)}]{farnia2018generalizable}
Farnia, F.; Zhang, J.; and Tse, D. 2019.
\newblock Generalizable Adversarial Training via Spectral Normalization.
\newblock In \emph{International Conference on Learning Representations
  (ICLR)}.

\bibitem[{Fazlyab et~al.(2019)Fazlyab, Robey, Hassani, Morari, and
  Pappas}]{NIPS2019_9319}
Fazlyab, M.; Robey, A.; Hassani, H.; Morari, M.; and Pappas, G. 2019.
\newblock Efficient and Accurate Estimation of Lipschitz Constants for Deep
  Neural Networks.
\newblock In \emph{Advances in Neural Information Processing Systems
  (NeurIPS)}.

\bibitem[{Golub and Van~der Vorst(2000)}]{golub2000eigenvalue}
Golub, G.~H.; and Van~der Vorst, H.~A. 2000.
\newblock Eigenvalue computation in the 20th century.
\newblock \emph{Journal of Computational and Applied Mathematics} 123(1-2):
  35--65.

\bibitem[{Goodfellow, Shlens, and Szegedy(2015)}]{goodfellow2014explaining}
Goodfellow, I.; Shlens, J.; and Szegedy, C. 2015.
\newblock Explaining and Harnessing Adversarial Examples.
\newblock In \emph{International Conference on Learning Representations
  (ICLR)}.

\bibitem[{Gouk et~al.(2018)Gouk, Frank, Pfahringer, and
  Cree}]{gouk2018regularisation}
Gouk, H.; Frank, E.; Pfahringer, B.; and Cree, M. 2018.
\newblock Regularisation of neural networks by enforcing lipschitz continuity.
\newblock \emph{arXiv preprint arXiv:1804.04368} .

\bibitem[{Gray et~al.(2006)}]{gray2006toeplitz}
Gray, R.~M.; et~al. 2006.
\newblock Toeplitz and circulant matrices: A review.
\newblock \emph{Foundations and Trends{\textregistered} in Communications and
  Information Theory} 2(3): 155--239.

\bibitem[{Green(1999)}]{green1999calculating}
Green, J. 1999.
\newblock Calculating the maximum modulus of a polynomial using Steckin's
  lemma.
\newblock \emph{SIAM journal on numerical analysis} 36(4): 1022--1029.

\bibitem[{Guti{\'e}rrez-Guti{\'e}rrez, Crespo
  et~al.(2012)}]{gutierrez2012block}
Guti{\'e}rrez-Guti{\'e}rrez, J.; Crespo, P.~M.; et~al. 2012.
\newblock Block Toeplitz matrices: Asymptotic results and applications.
\newblock \emph{Foundations and Trends{\textregistered} in Communications and
  Information Theory} 8(3): 179--257.

\bibitem[{He et~al.(2016)He, Zhang, Ren, and Sun}]{he2016deep}
He, K.; Zhang, X.; Ren, S.; and Sun, J. 2016.
\newblock Deep residual learning for image recognition.
\newblock In \emph{Proceedings of the IEEE Conference on Computer Vision and
  Pattern Recognition (CVPR)}.

\bibitem[{Huang et~al.(2017)Huang, Liu, Van Der~Maaten, and
  Weinberger}]{huang2017densely}
Huang, G.; Liu, Z.; Van Der~Maaten, L.; and Weinberger, K.~Q. 2017.
\newblock Densely connected convolutional networks.
\newblock In \emph{Proceedings of the IEEE Conference on Computer Vision and
  Pattern Recognition (CVPR)}.

\bibitem[{Iandola et~al.(2016)Iandola, Han, Moskewicz, Ashraf, Dally, and
  Keutzer}]{iandola2016squeezenet}
Iandola, F.~N.; Han, S.; Moskewicz, M.~W.; Ashraf, K.; Dally, W.~J.; and
  Keutzer, K. 2016.
\newblock SqueezeNet: AlexNet-level accuracy with 50x fewer parameters and< 0.5
  MB model size.
\newblock \emph{arXiv preprint arXiv:1602.07360} .

\bibitem[{Jain(1989)}]{jain1989fundamentals}
Jain, A.~K. 1989.
\newblock \emph{Fundamentals of digital image processing}.
\newblock Englewood Cliffs, NJ: Prentice Hall,.

\bibitem[{Krizhevsky, Sutskever, and Hinton(2012)}]{krizhevsky2012imagenet}
Krizhevsky, A.; Sutskever, I.; and Hinton, G.~E. 2012.
\newblock Imagenet classification with deep convolutional neural networks.
\newblock In \emph{Advances in Neural Information Processing Systems
  (NeurIPS)}.

\bibitem[{Latorre, Rolland, and Cevher(2020)}]{latorre2020lipschitz}
Latorre, F.; Rolland, P.; and Cevher, V. 2020.
\newblock Lipschitz constant estimation for Neural Networks via sparse
  polynomial optimization.
\newblock In \emph{International Conference on Learning Representations
  (ICLR)}.

\bibitem[{Lehoucq and Sorensen(1996)}]{lehoucq1996deflation}
Lehoucq, R.~B.; and Sorensen, D.~C. 1996.
\newblock Deflation techniques for an implicitly restarted Arnoldi iteration.
\newblock \emph{SIAM Journal on Matrix Analysis and Applications} 17(4):
  789--821.

\bibitem[{Li et~al.(2019)Li, Haque, Anil, Lucas, Grosse, and
  Jacobsen}]{NIPS2019_9673}
Li, Q.; Haque, S.; Anil, C.; Lucas, J.; Grosse, R.~B.; and Jacobsen, J.-H.
  2019.
\newblock Preventing Gradient Attenuation in Lipschitz Constrained
  Convolutional Networks.
\newblock In \emph{Advances in Neural Information Processing Systems
  (NeurIPS)}.

\bibitem[{Madry et~al.(2018)Madry, Makelov, Schmidt, Tsipras, and
  Vladu}]{madry2018towards}
Madry, A.; Makelov, A.; Schmidt, L.; Tsipras, D.; and Vladu, A. 2018.
\newblock Towards Deep Learning Models Resistant to Adversarial Attacks.
\newblock In \emph{International Conference on Learning Representations
  (ICLR)}.

\bibitem[{Miyato et~al.(2018)Miyato, Kataoka, Koyama, and
  Yoshida}]{miyato2018spectral}
Miyato, T.; Kataoka, T.; Koyama, M.; and Yoshida, Y. 2018.
\newblock Spectral normalization for generative adversarial networks.
\newblock \emph{arXiv preprint arXiv:1802.05957} .

\bibitem[{Pfister and Bresler(2018)}]{pfister2018bounding}
Pfister, L.; and Bresler, Y. 2018.
\newblock Bounding multivariate trigonometric polynomials with applications to
  filter bank design.
\newblock \emph{arXiv preprint arXiv:1802.09588} .

\bibitem[{Schmidt et~al.(2018)Schmidt, Santurkar, Tsipras, Talwar, and
  Madry}]{schmidt2018adversarially}
Schmidt, L.; Santurkar, S.; Tsipras, D.; Talwar, K.; and Madry, A. 2018.
\newblock Adversarially robust generalization requires more data.
\newblock In \emph{Advances in Neural Information Processing Systems
  (NeurIPS)}.

\bibitem[{Sedghi, Gupta, and Long(2019)}]{sedghi2018the}
Sedghi, H.; Gupta, V.; and Long, P.~M. 2019.
\newblock The Singular Values of Convolutional Layers.
\newblock In \emph{International Conference on Learning Representations
  (ICLR)}.

\bibitem[{Serra(1994)}]{serra1994preconditioning}
Serra, S. 1994.
\newblock Preconditioning strategies for asymptotically ill-conditioned block
  Toeplitz systems.
\newblock \emph{BIT Numerical Mathematics} 34(4): 579--594.

\bibitem[{Simonyan and Zisserman(2014)}]{simonyan2014very}
Simonyan, K.; and Zisserman, A. 2014.
\newblock Very deep convolutional networks for large-scale image recognition.
\newblock \emph{arXiv preprint arXiv:1409.1556} .

\bibitem[{Singla and Feizi(2019)}]{singla2019bounding}
Singla, S.; and Feizi, S. 2019.
\newblock Bounding Singular Values of Convolution Layers.
\newblock \emph{arXiv preprint arXiv:1911.10258} .

\bibitem[{Tramer et~al.(2020)Tramer, Carlini, Brendel, and
  Madry}]{tramer2020adaptive}
Tramer, F.; Carlini, N.; Brendel, W.; and Madry, A. 2020.
\newblock On adaptive attacks to adversarial example defenses.
\newblock \emph{arXiv preprint arXiv:2002.08347} .

\bibitem[{Tsuzuku, Sato, and Sugiyama(2018)}]{tsuzuku2018lipschitz}
Tsuzuku, Y.; Sato, I.; and Sugiyama, M. 2018.
\newblock Lipschitz-margin training: Scalable certification of perturbation
  invariance for deep neural networks.
\newblock In \emph{Advances in Neural Information Processing Systems
  (NeurIPS)}.

\bibitem[{Virmaux and Scaman(2018)}]{scaman2018lipschitz}
Virmaux, A.; and Scaman, K. 2018.
\newblock Lipschitz regularity of deep neural networks: analysis and efficient
  estimation.
\newblock In \emph{Advances in Neural Information Processing Systems
  (NeurIPS)}.

\bibitem[{Widom(1976)}]{widom1976asymptotic}
Widom, H. 1976.
\newblock Asymptotic behavior of block Toeplitz matrices and determinants. II.
\newblock \emph{Advances in Mathematics} 21(1): 1--29.

\bibitem[{Yi(2020)}]{yi2020asymptotic}
Yi, X. 2020.
\newblock Asymptotic Singular Value Distribution of Linear Convolutional
  Layers.
\newblock \emph{arXiv preprint arXiv:2006.07117} .

\bibitem[{Yoshida and Miyato(2017)}]{yoshida2017spectral}
Yoshida, Y.; and Miyato, T. 2017.
\newblock Spectral norm regularization for improving the generalizability of
  deep learning.
\newblock \emph{arXiv preprint arXiv:1705.10941} .

\bibitem[{Zagoruyko and Komodakis(2016)}]{zagoruyko2016wide}
Zagoruyko, S.; and Komodakis, N. 2016.
\newblock Wide residual networks.
\newblock \emph{arXiv preprint arXiv:1605.07146} .

\bibitem[{Zhang(2011)}]{zhang2011matrix}
Zhang, F. 2011.
\newblock \emph{Matrix theory: basic results and techniques}.
\newblock Springer Science \& Business Media.

\end{thebibliography}
